\documentclass[nohyperref]{article}

\usepackage{microtype}
\usepackage{graphicx}
\usepackage{subfigure}
\usepackage{booktabs} 

\usepackage{hyperref}

\usepackage[accepted]{icml2023}

\usepackage{amsmath}
\usepackage{amssymb}
\usepackage{mathtools}
\usepackage{amsthm}

\usepackage[capitalize,noabbrev]{cleveref}

\usepackage{amsmath,amsfonts,bm}
\usepackage{pifont}

\def\eqref#1{equation~\ref{#1}}

\def\1{\bm{1}}

\DeclareMathAlphabet{\mathsfit}{\encodingdefault}{\sfdefault}{m}{sl}
\SetMathAlphabet{\mathsfit}{bold}{\encodingdefault}{\sfdefault}{bx}{n}

\DeclareMathOperator*{\argmax}{arg\,max}

\newcommand{\tickYes}{\ding{51}}
\newcommand{\tickNo}{\ding{55}}

\usepackage{amsfonts} 
\usepackage{thmtools, thm-restate}

\usepackage{dsfont} 
\usepackage{pifont} 
\usepackage{makecell} 
\usepackage{ulem} 
\usepackage{wrapfig}
\usepackage{nameref} 
\usepackage{xcolor} 
\usepackage{enumitem} 

\theoremstyle{plain}

\theoremstyle{definition}
\newtheorem{definition}{Definition}

\newcommand\numberthis{\addtocounter{equation}{1}\tag{\theequation}}

\usepackage[textsize=tiny]{todonotes}

\newenvironment{sproof}{%
  \proof}{\endproof}

\icmltitlerunning{Variational Curriculum Reinforcement Learning for Unsupervised Discovery of Skills}

\begin{document}

\twocolumn[
\icmltitle{Variational Curriculum Reinforcement Learning \\for Unsupervised Discovery of Skills}



\icmlsetsymbol{equal}{*}

\begin{icmlauthorlist}
\icmlauthor{Seongun Kim}{equal,kaist}
\icmlauthor{Kyowoon Lee}{equal,unist}
\icmlauthor{Jaesik Choi}{kaist}
\end{icmlauthorlist}

\icmlaffiliation{kaist}{Kim Jaechul Graduate School of AI, KAIST}
\icmlaffiliation{unist}{Department of Computer Science and Engineering, UNIST}

\icmlcorrespondingauthor{Jaesik Choi}{jaesik.choi@kaist.ac.kr}

\icmlkeywords{Machine Learning, ICML}

\vskip 0.3in
]



\printAffiliationsAndNotice{\icmlEqualContribution} 

\begin{abstract}
Mutual information-based reinforcement learning (RL) has been proposed as a promising framework for retrieving complex skills autonomously without a task-oriented reward function through mutual information (MI) maximization or variational empowerment.
However, learning complex skills is still challenging, due to the fact that the order of training skills can largely affect sample efficiency. Inspired by this, we recast variational empowerment as curriculum learning in goal-conditioned RL with an intrinsic reward function, which we name Variational Curriculum RL (VCRL). From this perspective, we propose a novel approach to unsupervised skill discovery based on information theory, called Value Uncertainty Variational Curriculum (VUVC). We prove that, under regularity conditions, VUVC accelerates the increase of entropy in the visited states compared to the uniform curriculum. We validate the effectiveness of our approach on complex navigation and robotic manipulation tasks in terms of sample efficiency and state coverage speed. We also demonstrate that the skills discovered by our method successfully complete a real-world robot navigation task in a zero-shot setup and that incorporating these skills with a global planner further increases the performance.
\end{abstract}

\section{Introduction}

Intelligent creatures are able to efficiently explore the environments and learn useful skills in the absence of external supervision. By utilizing these skills, they can quickly accomplish tasks when they are later faced with specific tasks. To scale a learning agent to the real-world, it is crucial to achieve such ability of learning skills without supervision.
Recent studies on unsupervised RL suggest ways to alleviate the need for human effort.
Most of these approaches focus on reducing the burden of designing objective functions by incorporating intrinsic motivation objectives or leveraging concepts from information theory.
In this work, we further reconcile with the need not only to manually engineer objective functions but to craft the order of training skills.

Empowerment or MI-based RL \cite{klyubin2005empowerment, salge2014empowerment} has gained traction in recent years as a means of unsupervised skill discovery due to its intuitive interpretation and empirical successes \cite{eysenbach2018diversity, sharma2019dynamics, jabri2019unsupervised}. However, the common empowerment approach has been to either fix or parameterize the distribution of skills \cite{nair2018visual, pong2020skew, campos2020explore}. The efficiency of learning skills with respect to the number of required training samples is rather limited when the agent learns complex skills from a fixed skill distribution without an organized order.
The notion of \textit{curriculum} studies the effectiveness of the order of training skills. By selecting the order of appropriate skills, a learning agent may achieve a variety of complex skills \cite{florensa2018automatic, fang2019curriculum}. However, it is both necessary to define a set of tasks that can be used to generate curriculum \cite{klink2020self, zhang2020automatic} and specify a form of reward functions \cite{racaniere2019automated, ren2019exploration, narvekar2019learning}.

\let\thefootnote\relax\footnotetext{
\hspace{-\footnotesep-\footnotesep}
Codes are available at \href{https://github.com/seongun-kim/vcrl}{\url{github.com/seongun-kim/vcrl}}.
}

\begin{figure*}
    \centering
    \includegraphics[width=0.969\linewidth]{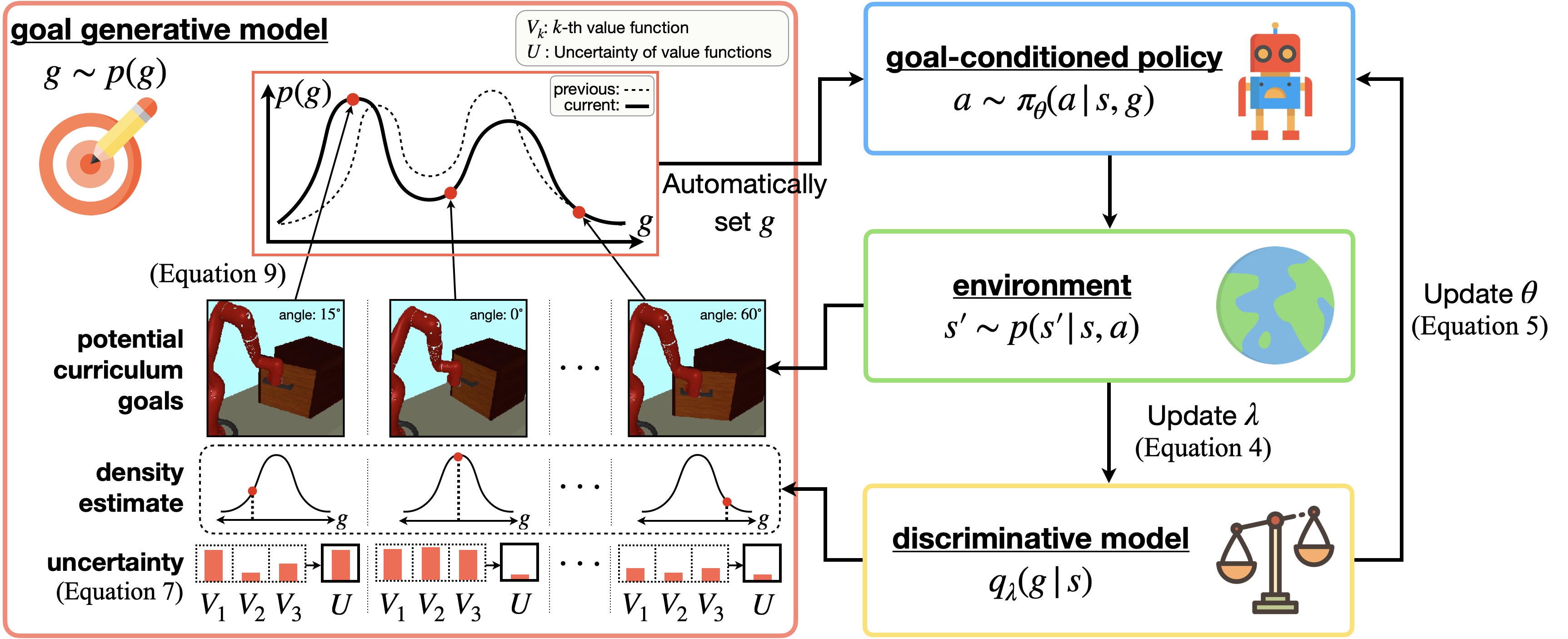}
    \caption{An overview of our proposed method, VUVC, under the unifying framework for curriculum learning in goal-conditioned RL. The value uncertainty proposes informative goals which would generate a stronger learning signal.
    The density estimate of potential curriculum goals indicates the novelty of the goal to the learning agent. The density estimate model is derived from the discriminative model, which is trained alongside the agent. This discriminative model provides intrinsic rewards to the agent. VUVC combines these two measures to construct a goal generative model, promoting unsupervised exploration of the entire state space by the agent.}
    \label{fig:main_figure}
\end{figure*}

To rectify this issue, we interpret empowerment as a unifying framework for curriculum learning in goal-conditioned RL (GCRL). Recasting variational empowerment as curriculum learning in GCRL with intrinsic reward function, interestingly our Variational Curriculum RL (VCRL) framework encapsulates most of the prior MI-based approaches \cite{nair2018visual, pong2020skew, campos2020explore}. In this regard, we derive a new approach to information-theoretic skill discovery, Value Uncertainty Variational Curriculum (VUVC) that allows us to automatically generate curriculum goals which maximize the expected information approximated as the uncertainty in predictions of an ensemble of value functions. We analyze asymptotic behavior of the entropy of visited states and provide the reasons 
why our method results in much faster coverage of the state space compared to existing methods.

The main contributions of this paper can be summarized as follows: (1) We provide the unifying framework VCRL encapsulating most of the prior MI-based approaches. (2) We propose VUVC, a value uncertainty based approach to information-theoretic skill discovery, aimed at automatically generating curricula for training skills and which is supported by theoretical justification. (3) We show the effectiveness of our approach on complex navigation, robotic manipulation in both configuration and image state space, and real-world robotic navigation tasks and illustrate that the skills discovered by our method can be further improved by incorporating them with a global planner.

\section{Background}
\label{sec:background}

\subsection{Goal-Conditioned Reinforcement Learning}
Goal-conditioned RL \cite{kaelbling1993learning} extends the standard RL framework to enable agents to accomplish a variety of tasks.
It solves the problem formulated as a goal-conditioned Markov decision process (MDP) which is defined as a tuple $\langle \mathcal{S}, \mathcal{G}, \mathcal{A}, P, R_g, \gamma \rangle$, where $\mathcal{S}$ is the set of states, $\mathcal{G}$ is the set of goals, $\mathcal{A}$ is the set of actions, $P: \mathcal{S} \times \mathcal{A} \times \mathcal{S} \to [0, +\infty)$ is the transition probability, $R_g: \mathcal{S} \times \mathcal{A}\to \mathbb{R}$ is the goal-conditioned reward function and $\gamma \in [0, 1]$ is the discount factor.
The objective of GCRL is to find the policy $\pi_\theta(a|s, g)$ parameterized with $\theta$ where $s \in \mathcal{S}, a \in \mathcal{A}, g \in \mathcal{G}$ and $\pi: \mathcal{S} \times \mathcal{A} \times \mathcal{S} \to [0, +\infty)$ that maximizes the universal value function \cite{schaul2015universal}:
\begin{align}
\label{eq:gcrl_objective}
    \theta
    &\gets \argmax_\theta V^{\pi_\theta}(s, g) \triangleq \nonumber \\
    &\mathbb{E}_{\substack{a_t \sim \pi_\theta(a_t | s_t, g), \\s_{t+1} \sim P(s_{t+1} | s_t, a_t)}} \Bigg[\sum_{t=0}^{\infty} \gamma^t R_g(s_t, a_t) \Big| s_0 = s \Bigg].
\end{align}


\begin{table*}[t]
    \centering
    \begin{tabular}{c|ccc}
        \toprule
        Methods & \textbf{$q_{\lambda}(g|s)$} & $p(g)$ & \makecell{Non-stationary \\ goal distribution} \\ 
        \midrule
        GCRL (w/ sparse reward) & $\frac{1}{Z} \exp(1-2\delta_g\mathcal{U}_{[s\pm\delta_g]})$ & $p^{\mathrm{target}}(g)$ & \tickNo \\
        GCRL (w/ dense reward) & $\mathcal{N}(s, \sigma^2 I)$ & $p^{\mathrm{target}}(g)$ & \tickNo \\
        EDL \cite{campos2020explore} & $\mathcal{N}(\mu(s), \sigma^2 I)$ & $p^{\mathrm{explored}}(g)$ & \tickNo \\
        RIG \cite{nair2018visual} & $\mathcal{N}(\mu(s), \sigma^2 I)$ & $p_t^{\mathrm{visited}}(g)$ & \tickYes \\
        Skew-Fit \cite{pong2020skew} & $\mathcal{N}(\mu(s), \sigma^2 I)$ & $\propto p_t^{\mathrm{visited}}(g)^{\alpha}$ & \tickYes \\
        VUVC (\textbf{ours}) & $\mathcal{N}(\mu(s), \sigma^2 I)$ & $\propto U(g)p_t^{\mathrm{visited}}(g)^{\alpha}$ & \tickYes \\
        \bottomrule
    \end{tabular}
    \caption{Variants of VCRL framework which encapsulate most of the prior MI-based methods, depending on the choice of a discriminator $q_\lambda(g|s)$, a goal generative model $p(g)$, and whether $p(g)$ is stationary or not, where both $q_\lambda(g|s)$ and $p(g)$ are components of the MI objective. The discriminator determines the shape of goal-conditioned reward functions including sparse and dense shapes.}
    \label{tab:vcrl}
\end{table*}

\subsection{Mutual Information and Empowerment}

In the context of RL, MI maximization such as empowerment generally means maximizing the mutual information between a function of states and a function of actions to learn latent-conditioned policies $\pi(a|s, z)$ where the latent code $z$ can be interpreted as a macro-action, skill or goal \cite{eysenbach2018diversity, sharma2019dynamics}. Empowerment maximizes the following MI objective:
%
\begin{align}
\label{eq:mutual_information}
    \mathcal{I}(s;z) 
    &= \mathcal{H}(s) - \mathcal{H}(s|z) \nonumber \\
    &= \mathcal{H}(z) - \mathcal{H}(z|s) \nonumber \\
    &= \mathbb{E}_{z\sim p(z), s\sim p(s|z)}\left[ \log p(z|s) - \log p(z) \right] \nonumber \\
    &\ge \mathbb{E}_{z\sim p(z), s\sim p(s|z)}\left[ \log q_\lambda(z|s) - \log p(z) \right], 
\end{align}
where $\mathcal{H}(\cdot)$ is the Shannon entropy, $p(z)$ is the prior distribution, and $ q_\lambda(z|s)$ represents the variational approximation for intractable posterior $p(z|s)$ parameterized with $\lambda$, often called a discriminator \cite{eysenbach2018diversity, sharma2019dynamics, campos2020explore}. This objective provides a way to train a policy that guides agents to explore diverse states by maximizing $\mathcal{H}(s)$ and makes the state $s$ distinguishable from the latent code $z$ by minimizing $\mathcal{H}(s|z)$.






\section{Variational Curriculum Reinforcement Learning}
\label{sec:vcrl}

To recast the aforementioned MI-based RL as VCRL, we first present that general GCRL methods optimize empowerment objective by formulating a discriminator to represent commonly used goal-conditioned reward functions. We then expand this setting to a curriculum learning framework with a goal generative model, which we name VCRL where Table \ref{tab:vcrl} summarizes variants of the VCRL framework.

Henceforth, we consider the latent code $z$ in Equation \ref{eq:mutual_information} as a goal $g$ and assume the goal space matches the state space, while VCRL framework is not limited to this assumption and trivially extended by introducing a state abstraction function \cite{ren2019exploration}. The objective now becomes equivalent to that of a GCRL where the resulting policy aims to reach $g$ \cite{pong2020skew, choi2021variational}. Given a policy $\pi_\theta(a|s, g)$ and a discriminator $q_\lambda(g|s)$, an objective of MI-based RL is to maximize a variational lower bound:
\begin{align}
\label{eq:mi_objective}
    \mathcal{F}(\theta, \lambda) = \mathbb{E}_{\substack{g \sim p(g), \\s \sim \rho^\pi(s|g)}} [\log q_\lambda(g|s) - \log p(g)],
\end{align}
where $\rho^\pi(s|g)$ is a stationary state distribution induced by the goal-conditioned policy $\pi(a|s, g)$ \cite{gregor2016variational, campos2020explore}. To solve this joint optimization problem, we iteratively fix one parameter and optimize the other one at each training epoch $i$:
\begin{align}
    \label{eq:discriminator_objective}
    \lambda^{(i)} &\gets \argmax_{\lambda} \mathbb{E}_{\substack{g \sim p(g), \\s \sim \rho^{\pi_{\theta^{(i-1)}}}(s|g)}} [\log q_\lambda(g|s) - \log p(g)] \\
    \label{eq:policy_objective}
    \theta^{(i)} &\gets \argmax_{\theta} \mathbb{E}_{\substack{g \sim p(g), \\s \sim \rho^{\pi_{\theta}}(s|g)}} [\log q_{\lambda^{(i)}}(g|s)].
\end{align}

As described in the prior work \cite{warde-farley2019unsupervised, choi2021variational}, it has been shown that Equation \ref{eq:policy_objective} which is also called an intrinsic reward \cite{gregor2016variational}, recovers the objective of GCRL in Equation \ref{eq:gcrl_objective} with dense rewards. By choosing a Gaussian distribution with mean $s$ and fixed variance $\sigma^2 I$ for $q_\lambda(g|s)$ where $I$ is the identity matrix, this objective becomes a negative $l_2$ distance between $s$ and $g$. Similarly, one can show that the intrinsic reward represented in Equation \ref{eq:policy_objective} becomes a sparse reward where an agent gets $0$ reward if $l_2$ distance between $s$ and $g$ is within some threshold $\delta_g$ and gets $-1$ otherwise. Other MI-based methods can also be considered a GCRL by modeling $q_\lambda(g|s)$ to follow $\mathcal{N}(\mu(s), \sigma^2 I)$ where $\mu(s)$ is a function approximator usually following an encoder structure.


\begin{figure*}
    \centering
    \includegraphics[width=0.8\linewidth]{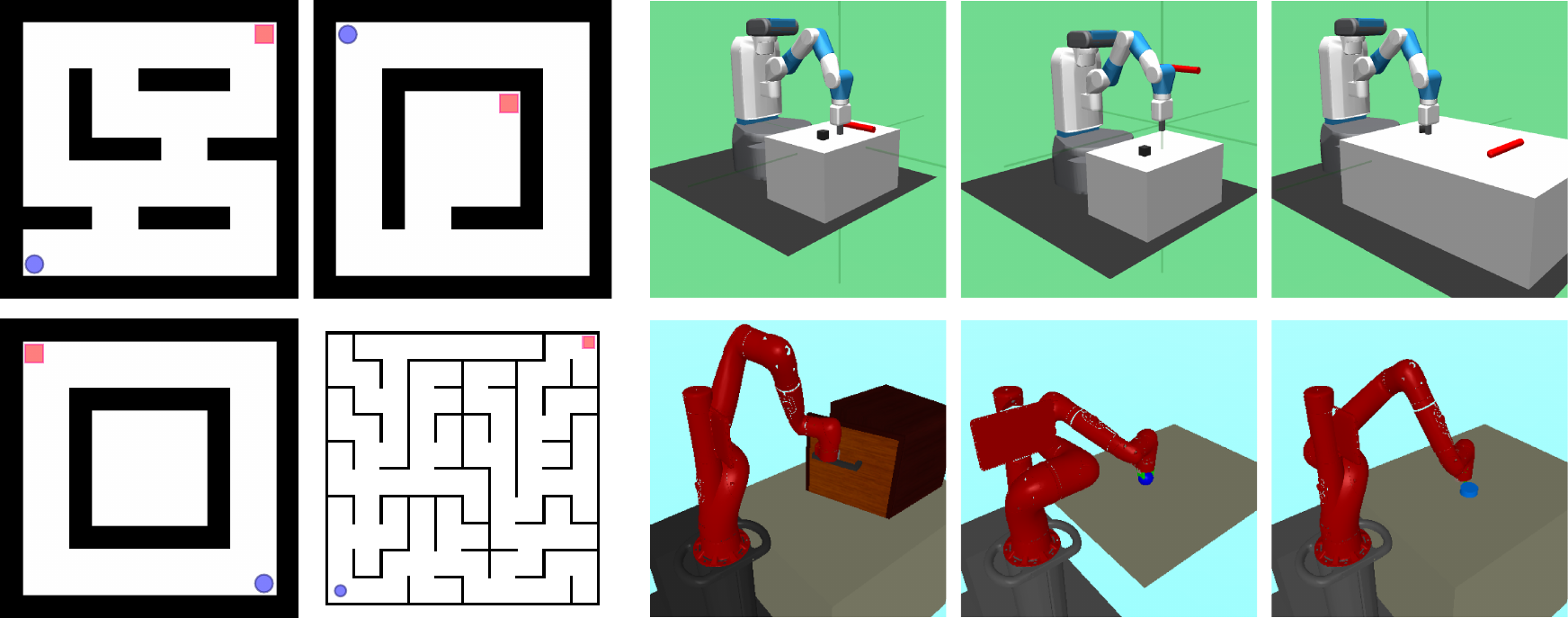}
    \caption{Illustrations of simulated environments. (Left) Point maze navigation tasks which we name \textit{PointMazeA}, \textit{B}, \textit{C}, and \textit{SquareLarge} in sequential order. The initial state and goal distribution of each task are depicted by a blue circle and red box, respectively. (Top right) Configuration-based robot manipulation tasks: \textit{FetchPush},  \textit{FetchPickAndPlace}  and \textit{FetchSlide}. The goal distribution which represents the target position for the puck, is illustrated by a red cylinder. (Bottom right) Vision-based robot manipulation tasks: \textit{SawyerDoorHook}, \textit{SawyerPickup} and \textit{SawyerPush}.}
    \label{fig:environments}
    \vspace{-0.3cm}
\end{figure*}

We further expand the interpretation of MI-based methods as a framework of GCRL to a framework of curriculum learning, which we term VCRL. Curriculum learning in RL studies the order of training skills or tasks. In the context of GCRL, the order of tasks, \textit{curriculum}, is determined by characterizing a goal distribution $p(g)$ \cite{fournier2018accuracy, florensa2018automatic, racaniere2019automated, ren2019exploration, zhang2020automatic, klink2020self}.
Without an explicit design of $p(g)$, VCRL is reduced to a simple GCRL where a target goal is given from the environment, $p^{\mathrm{target}}(g)$. Otherwise, one can design a goal generative model to satisfy various purposes of the training. For instance, EDL \cite{campos2020explore}, a variant of MI-based RL, aims to train a state space covering skill. EDL first learns $p^{\mathrm{explored}}(g)$ along with an exploration policy \cite{lee2019efficient} which tries to cover the entire state space. Then, it optimizes the MI objective (Equation \ref{eq:mi_objective}) with the stationary goal distribution $p^{\mathrm{explored}}(g)$. Skew-Fit \cite{pong2020skew} also seeks to learn a state space covering skill in an unsupervised manner. However, unlike EDL, it assumes a non-stationary goal distribution to ensure that the state density $p(s)$ converges to uniform distribution. This is achieved by formulating the goal distribution, $p(g)$, to be proportional to the approximate state density, $p^{\mathrm{visited}}(s)$, raised to a skewing parameter $\alpha$ within the range of $[-1, 0)$. Similarly, RIG samples goals directly from $p^{\mathrm{visited}}(s)$.



\vspace{-0.1cm}
\section{Value Uncertainty Variational Curriculum}



Despite the many empirical successes of empowerment methods, learning complex skills is still challenging since there has been little consideration of $p(g)$ in the MI objective \cite{achiam2018variational, eysenbach2018diversity, warde-farley2019unsupervised, campos2020explore}.
To efficiently learn complex skills, it is important to effectively optimize the variational empowerment in Equation \ref{eq:mi_objective}.
To this end, the agent should seek out goals from which it can learn the most. This can be formalized in the uncertainty of value functions which track the performance of the policy. To estimate the uncertainty, we use an ensemble of multiple value functions that has been widely adopted in the literature with empirical success \cite{osband2016deep, lakshminarayanan2017simple, osband2018randomized, zhang2020automatic}. Formally, we maintain an ensemble of parameters for value functions: $\mathbb{\psi}=\{\psi_1, ..., \psi_K\}$, which is randomly initialized independently,
\begin{align}
    \text{Value functions   } v_{\psi} : s, g \rightarrow V_{\psi}(s, g). 
\end{align}
%
We quantify the uncertainty of value functions in predictions of the ensemble members from the initial state by computing the variance over the ensemble of the value functions:
%
\begin{align}
    \label{eq:uncertainty}
    \text{Uncertainty   } U(g) : \mathrm{Var}[\{ V_{\psi}(s_0, g) | \psi \in \{\psi_1, ..., \psi_K\}].
\end{align}
%
\begin{restatable}{proposition}{Firstprop}\label{prop_1}
If $V_{\psi}(s_0, g)$ follows a log-concave distribution, then we have
\begin{align}
    \mathcal{I}(V_{\psi}(s_0, g);\psi| s_0, g) \ge \log (2\sqrt{\mathrm{Var}[V_{\psi}(s_0, g)]}).
\end{align}
\end{restatable}


\begin{sproof}
We rewrite the mutual information as the difference between conditional entropy and marginal entropy. We then use the result in \cite{marsiglietti2018lower} on a lower bound on the entropy of a log-concave random variable, expressed in terms of the $p$-th absolute moment to obtain the conclusion. The complete proof appears in Appendix \ref{A_Proofs}.
\end{sproof}


It follows from Proposition \ref{prop_1} that finding a goal which maximizes the mutual information can be relaxed into the surrogate problem, which is to select a goal that maximizes the uncertainty in predictions of an ensemble of value functions when we take $K \to \infty$. With this intuition, one natural option to sample goals is to compute a goal probability proportional to the uncertainty $p(g) \propto U(g)$, where $g \in \mathrm{support}(p_t^{\mathrm{visited}})$. To prevent goals with lower density from being frequently proposed, we adopt the Skew strategy \cite{pong2020skew} which assigns more weight to rare samples by skewing the goal sampling probability. We therefore sample goals from the following distribution:
%
\begin{align}
\label{eq:vuvc_probs}
    p_t^{\mathrm{VUVC}}(g) = \frac{1}{Z_{t, \alpha}}U(g)p_t^{\mathrm{visited}}(g)^{\alpha},  \quad\alpha \in [-1, 0),
\end{align}
where $Z_{t,\alpha}$ is the normalizing coefficient. We approximate $p_t^{\mathrm{visited}}$ by training a generative model on samples in the replay buffer, where we use a $\beta$-VAE \cite{higgins2017betavae} in our experiments. 
We term a VCRL method with a goal generative model following Equation \ref{eq:vuvc_probs} as VUVC.

\begin{figure*}
    \centering
    \includegraphics[width=0.95\linewidth]{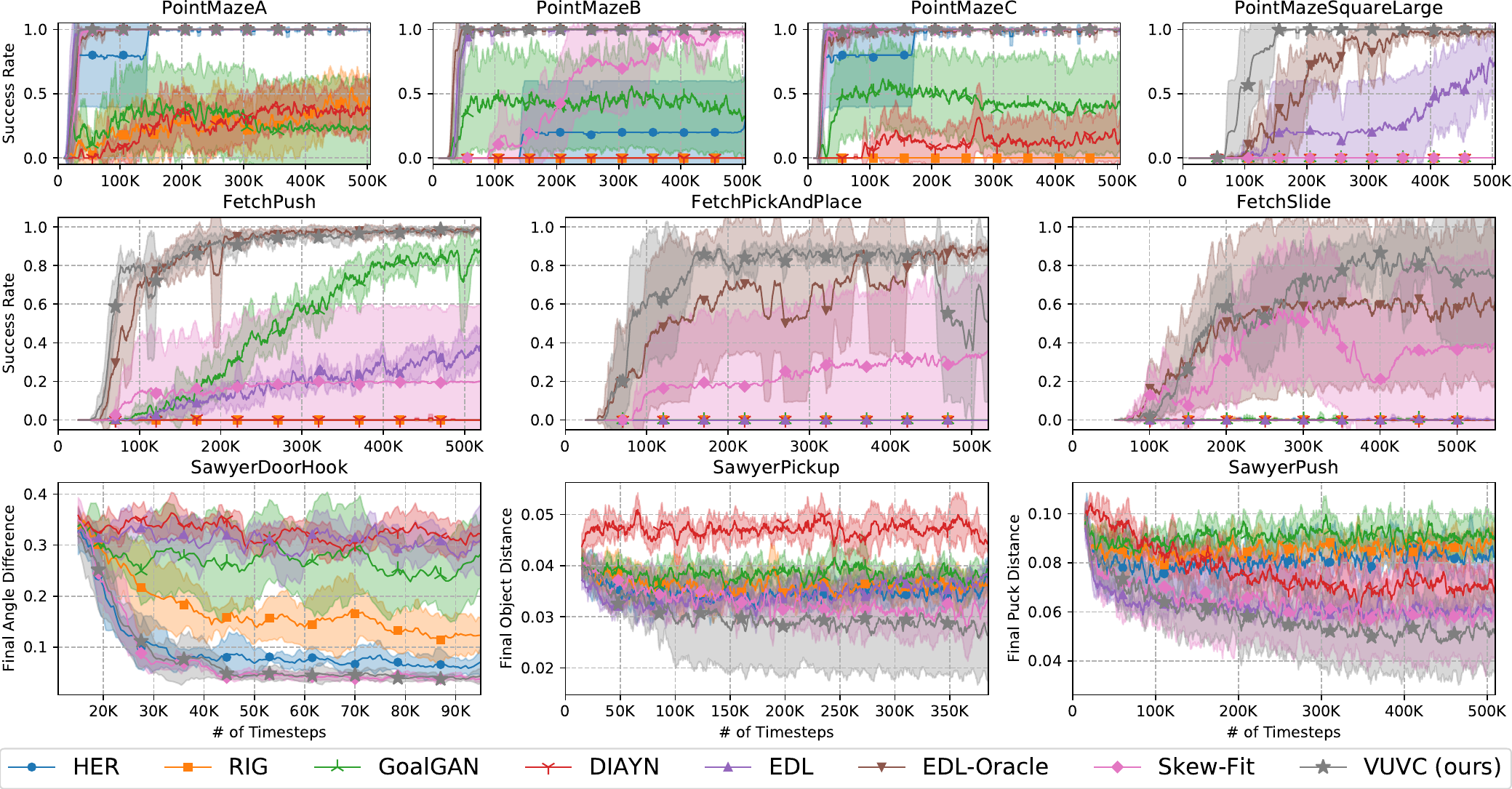}
    \caption{Learning curves for configuration-based point maze navigation tasks (top), continuous robot control tasks (middle), and vision-based continuous robot manipulation tasks (bottom). \textit{Mean (SD)} of each performance measure over 5 random seeds are reported where results are smoothed across 10 training epochs for each seed. VUVC consistently outperforms other VCRL variants for all tasks.}
    \label{fig:quantitative}
\end{figure*}

\begin{definition}(Expected Entropy Increment over Uniform Curriculum).\label{def:def_1}
Given the empirical distribution of the visited state
\begin{equation}
    p_t^{\mathrm{visited}}(s)=\sum_{i=1}^t\frac{\mathbb{I}(s_i=s)}{t},
\end{equation}
where $\mathbb{I}(\cdot)$ is an indicator function,
uniform curriculum goal distribution $p_t^{\mathcal{U}}$ and value uncertainty-based curriculum goal distribution $p_t^{\mathrm{VU}}$ are defined as follows:
\begin{align}
    p_t^{\mathcal{U}}(g) &= \mathcal{U}(\mathrm{support}(p_t^{\mathrm{visited}}))(g), \\
    p_t^{\mathrm{VU}}(g) &= \frac{1}{Z_t}U(g)p_t^{\mathcal{U}}(g),
\end{align}
where $Z_t$ is the normalizing coefficient, $p_t^{\mathcal{U}}$ is uniform over the support of the $p_t^{\mathrm{visited}}$ and $U(g)$ is the value uncertainty. Then the expected entropy increment over uniform curriculum $I_t$ is defined as 
\begin{align}
\label{eq:entropy_increment}
    I_t = \mathbb{E}_{g\sim p_t^{\mathrm{VU}}}[\mathcal{H}(p_{t+1}^{\mathrm{visited}})] - \mathbb{E}_{g\sim p_t^{\mathcal{U}}}[\mathcal{H}(p_{t+1}^{\mathrm{visited}})].
\end{align}
\end{definition}

To study the asymptotic behavior of the expected next step entropy induced by VUVC, we define the expected entropy increment over uniform curriculum in Equation \ref{eq:entropy_increment} for the case of discrete state space. However, computing the empirical distribution of the next visited state $p_{t+1}^{\mathrm{visited}}$ requires marginalizing out the MDP dynamics which is intractable to compute. Therefore, we consider two special cases when (1) an agent always reaches the goal in Proposition \ref{prop_2} and (2) an agent sometimes fails to reach goals but potentially increases the amount of entropy in Proposition \ref{prop_3}.

\begin{restatable}{proposition}{Secondprop}\label{prop_2}
Given $\epsilon=\frac{1}{t}$ and $\rho^{\pi_{\theta}}(s|g)=\mathbb{I}(s=g)$, if 
\begin{align}
    \mathrm{Cov}[U(g),\log p_t^{\mathrm{visited}}(g)] \leq 0,
\end{align}
and take $\epsilon \to 0$, then we have,                                                   
\begin{align}
    &\lim_{\epsilon \to 0}\frac{\partial}{\partial \epsilon}I_t = \nonumber \\
    &\lim_{\epsilon \to 0}\frac{\partial}{\partial \epsilon} \left( \mathbb{E}_{g\sim p_t^{\mathrm{VU}}}[\mathcal{H}(p_{t+1}^{\mathrm{visited}})] - \mathbb{E}_{g\sim p_t^{\mathcal{U}}}[\mathcal{H}(p_{t+1}^{\mathrm{visited}})] \right) > 0.
\end{align}
\end{restatable}

\begin{sproof}
We begin by deriving a next step empirical distribution of the visited state given a curriculum goal $g$ and a stationary state distribution induced by the policy $\rho^{\pi_{\theta}}(s|g)$, which can be written as $p_{t+1}^{\mathrm{visited}}(s)=\frac{p_{t}^{\mathrm{visited}}(s) + \epsilon\rho^{\pi_{\theta}}(s|g)}{1+\epsilon}$. Plugging this back into Definition \ref{def:def_1}, we analyze asymptotic behavior of the expected entropy increment and obtain the conclusion with the assumption $\rho^{\pi_{\theta}}(s|g)=\mathbb{I}(s=g)$. The complete proof is provided in Appendix \ref{A_Proofs}.
\end{sproof}


With an accurate goal-conditioned policy and the model of dynamics, Proposition \ref{prop_2} gives us intuition that our VUVC is at least better than the uniform curriculum which Skew-Fit aims to converge to, if the uncertainty of the learned value functions $U(g)$ and the log density of $p_t^{\mathrm{visited}}$ are negatively correlated. We expect this negative correlation to happen frequently, since the uncertainty is positive for novel states, but it eventually reduces to zero with a sufficiently large number of samples.

\begin{restatable}{proposition}{Thirdprop}\label{prop_3}

Define the set $\mathcal{G}=\mathcal{G}_{\mathrm{exploit}} \cup \mathcal{G}_{\mathrm{uninfo}} \cup \mathcal{G}_{\mathrm{info}}$ and positive constant $\Delta_1, \Delta_2$ where
\begin{align}
& \rho^{\pi_{\theta}}(s|g) = \begin{cases} \mathbb{I}(s=g) &  \text{for } g \in \mathcal{G}_{\mathrm{exploit}} \\ \rho_{\mathrm{uninfo}}^{{\pi_{\theta}}}(s|g) &  \text{for } g \in  \mathcal{G}_{\mathrm{uninfo}} \\ \rho_{\mathrm{info}}^{{\pi_{\theta}}}(s|g) &  \text{for } g \in  \mathcal{G}_{\mathrm{info}},\end{cases}
\end{align}
for all $g \in \mathcal{G}_{\mathrm{uninfo}}$,
\begin{align*}
& \mathbb{E}_{s\sim\rho_{\mathrm{uninfo}}^{{\pi_{\theta}}}(s|g)}[\log p_t^{\mathrm{visited}}(s)] = \log p_t^{\mathrm{visited}}(g) + \Delta_1,
\end{align*}
and for all $g \in \mathcal{G}_{\mathrm{info}}$, 
\begin{align*}
& \mathbb{E}_{s\sim\rho_{\mathrm{info}}^{{\pi_{\theta}}}(s|g)}[\log p_t^{\mathrm{visited}}(s)] = \log p_t^{\mathrm{visited}}(g) - \Delta_2.
\end{align*}
Given $\epsilon=\frac{1}{t}$, if 
\begin{align*}
    \mathrm{Cov}[U(g),\log p_t^{\mathrm{visited}}(g)] &\leq 0, \\
    \mathbb{E}_{g \in \mathcal{G}_{\mathrm{uninfo}}}[p_t^{\mathrm{VU}}(g)] &\leq \mathbb{E}_{g \in \mathcal{G}_{\mathrm{uninfo}}}[p_t^{\mathcal{U}}(g)], \\
    \mathbb{E}_{g \in \mathcal{G}_{\mathrm{info}}}[p_t^{\mathrm{VU}}(g)] &\geq \mathbb{E}_{g \in \mathcal{G}_{\mathrm{info}}}[p_t^{\mathcal{U}}(g)], 
\end{align*}
and take $\epsilon \to 0$, then we have,
\begin{align*}
    &\lim_{\epsilon \to 0}\frac{\partial}{\partial \epsilon}I_t = \\
    &\lim_{\epsilon \to 0}\frac{\partial}{\partial \epsilon} \left( \mathbb{E}_{g\sim p_t^{\mathrm{VU}}}[\mathcal{H}(p_{t+1}^{\mathrm{visited}})] - \mathbb{E}_{g\sim p_t^{\mathcal{U}}}[\mathcal{H}(p_{t+1}^{\mathrm{visited}})] \right) > 0.
\end{align*}
\end{restatable}

\begin{sproof}
The proof proceeds in a similar manner as Proposition \ref{prop_2} except for an assumption $\mathcal{G}=\mathcal{G}_{\mathrm{exploit}} \cup \mathcal{G}_{\mathrm{uninfo}} \cup \mathcal{G}_{\mathrm{info}}$. The complete proof is in Appendix \ref{A_Proofs}.
\end{sproof}


%


Proposition \ref{prop_3} extends Proposition \ref{prop_2} to the case where the goal-conditioned policy is sub-optimal and fails to achieve some of the goals. It implies that we need a curriculum method which can filter out uninformative states when the policy can not consistently achieve certain states, in order to achieve a rapid increment of entropy. Empirical observations indicate that VUVC achieves this effect (further details provided in Section \ref{sec:experiments}).


\section{Experiments}
\label{sec:experiments}
\begin{figure}
    \centering
    \includegraphics[width=\linewidth]{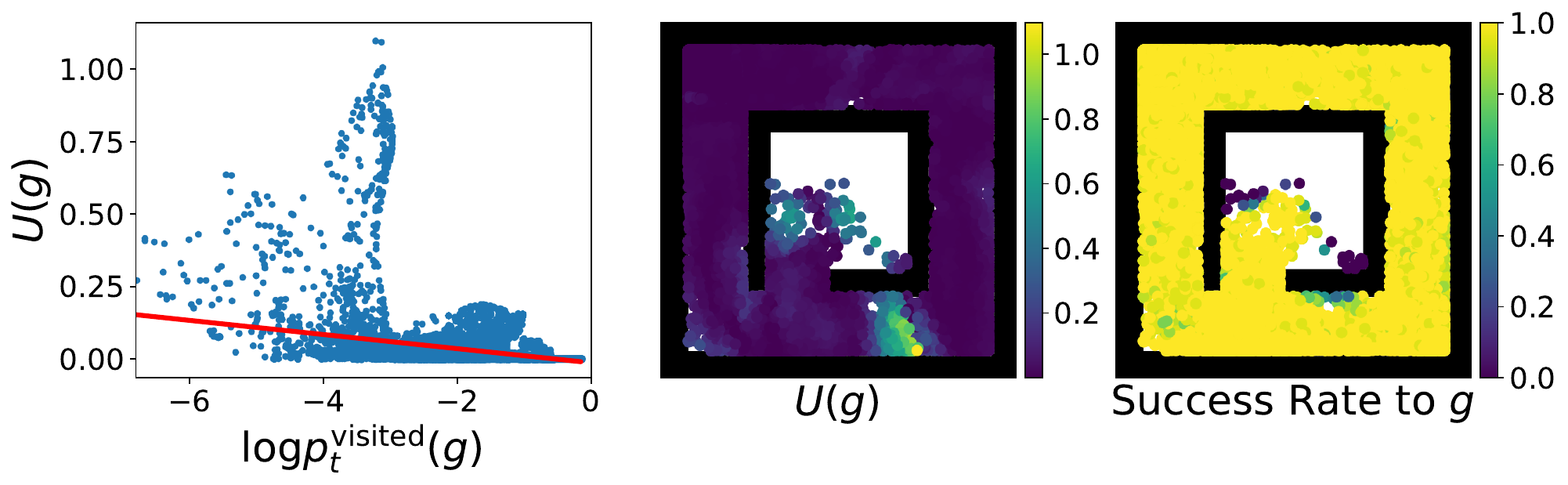}
    \vspace{-0.8cm}
    \caption{An illustration of the relation between value uncertainty and log density of visited states (left) and the landscape of value uncertainty (middle) and success rate (right).}
    \vspace{-0.25cm}
    \label{fig:prop}
\end{figure}


\subsection{Experimental Setup and Baselines}

\begin{figure}[ht]
    \centering
    \includegraphics[width=\linewidth]{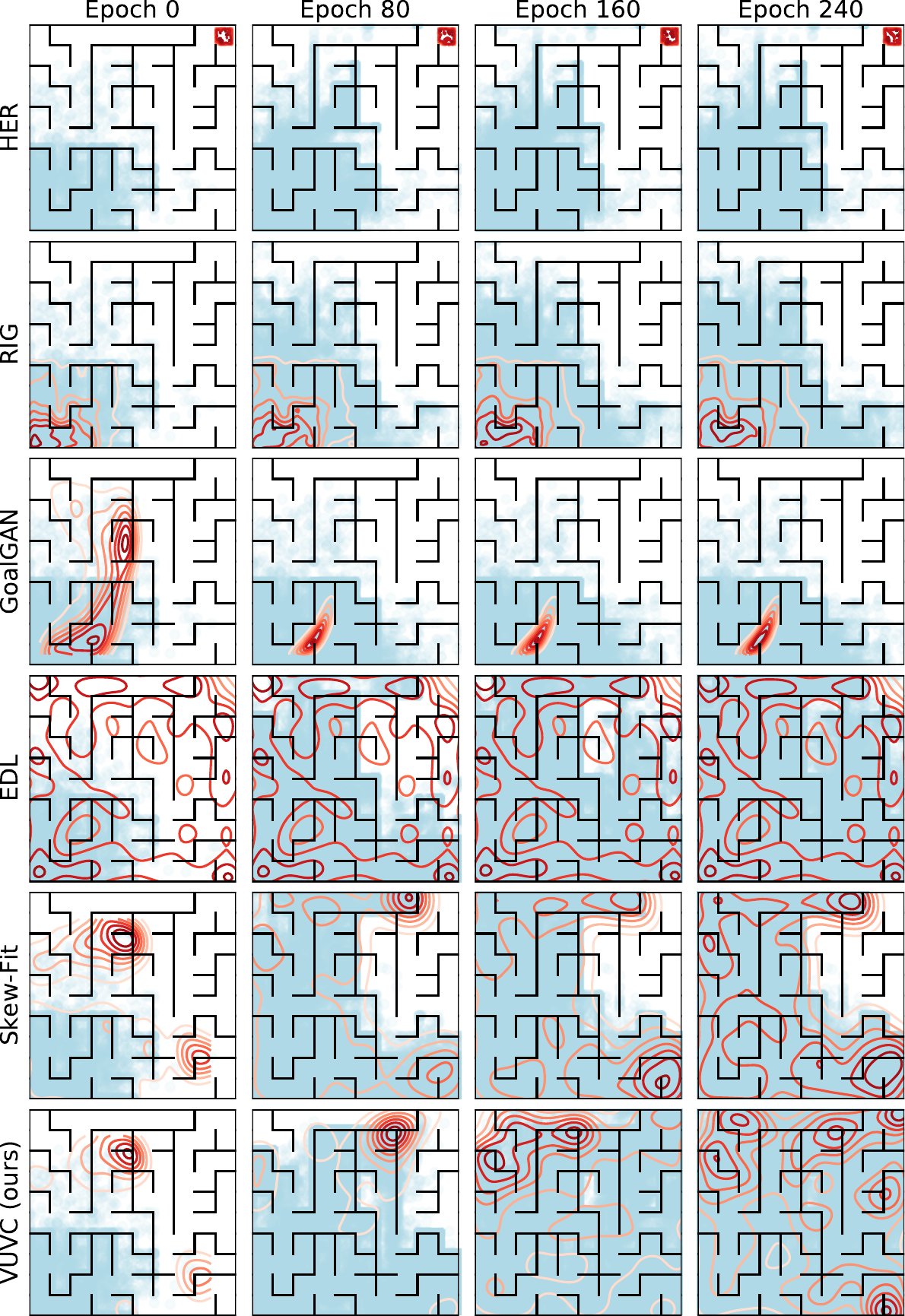}
    \vspace{-0.5cm}
    \caption{Curriculum goal distribution and accumulated visited states. The red contour line illustrates the curriculum goal distributions and cyan dots represent visited states by the agent. VUVC covers the state space significantly faster than the baselines.}
    \label{fig:density}
    \vspace{-0.5cm}
\end{figure}

We validate the effectiveness of VUVC on 10 different environments. They consist of point maze navigation tasks \cite{zhang2020automatic, trott2019keeping}, configuration-based robot control tasks \cite{plappert2018multi}, and vision-based robot manipulation tasks \cite{nair2018visual} which are shown in Figure \ref{fig:environments}. Especially, for configuration-based robot tasks, we modify the initial state and goal distribution following from the prior work \cite{ren2019exploration} to consider more complicated tasks which require extensive exploration.
Further details of experimental setups are presented in Appendix \ref{B_Exp_Details}.
By comparing VUVC to HER \cite{andrychowicz2017hindsight}, we study how effectively explicit curriculum improves sample efficiency over implicit curriculum.
We examine how well value uncertainty curriculum goals encourages exploration 
over goals from GoalGAN \cite{florensa2018automatic} which generates goals by measuring task difficulty through success rate, over goals from DIAYN \cite{eysenbach2018diversity} which divides the visited state space into separate sections for each skill, 
or over goals from RIG \cite{nair2018visual} and Skew-Fit \cite{pong2020skew} which sample goals from the density estimate.
We also investigate the importance of gradually increasing state coverage for the goal distribution by comparing it to EDL \cite{campos2020explore}, and investigate how efficiently VUVC increases the visited state entropy.


\subsection{Comparison of Sample Efficiency}
We compare the number of required samples for task completion in various environments which are based on either configuration observation or image observation. Our experimental results illustrated in Figure \ref{fig:quantitative} show that VUVC outperforms a variety of VCRL variants. Note that although EDL and EDL-Oracle take advantage of an additional training phase, VUVC outperforms them.

\vspace{-0.1cm}
\paragraph{Point Maze Navigation Tasks} 
VUVC successfully accomplishes all tasks, while some baseline methods fail.
Especially in the complicated \textit{PointMazeSquareLarge} environment, VUVC requires much less interaction for task completion. This result suggests the importance of an elaborate curriculum goal distribution in comparison to GoalGAN or Skew-Fit and emphasizes the importance of a gradually increasing state covering goal distribution when compared to EDL and EDL-Oracle.

\vspace{-0.1cm}
\paragraph{Configuration-based Robotic Manipulation Tasks}
In all three tasks, VUVC significantly outperforms all baselines. It is also noteworthy that VUVC performs better than EDL-Oracle, even though our method does not make an excessive assumption (i.e., the need for an oracle uniform goal sampler).
In comparison to Skew-Fit which also generates goals from a non-stationary distribution, the success rate of VUVC increases much faster. This result indicates to us that our method increases the entropy of the visited state distribution more efficiently than Skew-Fit.

\vspace{-0.1cm}
\paragraph{Vision-based Robotic Manipulation Tasks}
VUVC presents the best performance compared to other VCRL variants in image observation environments. We train a policy in a latent space instead of directly training in an image space, as it has been shown that this solves RL problems in an image space efficiently \cite{nair2018visual}, where an encoder of state density estimate model for a goal generator is used for a mapping function from an image observation to a latent observation. Even in a poorly-structured observation space, Figure \ref{fig:quantitative} shows that VUVC consistently outperforms a variety of baseline methods. Note that DIAYN struggles in the \textit{SawyerDoorHook} and \textit{SawyerPickup} tasks as its policy remain close to the initial state during the training phase.




\subsection{Impact of the Value Uncertainty}
\label{sec:impact_of_uncertainty}
To see the effects of the value uncertainty in the curriculum, in Figure \ref{fig:prop}, we investigate (1) how the value uncertainty $U(g)$ and log density of visited states $p_t^{\mathrm{visited}}$ are correlated, and (2) how well the value uncertainty filters out uninformative states. In general, we observe that $U(g)$ and $p_t^{\mathrm{visited}}$ show negative correlation, indicating that VUVC covers the state space faster than the uniform curriculum for the optimal goal-conditioned policy as the regularity condition of Proposition \ref{prop_2} holds empirically.
In addition to this, we observe a case which satisfies the regularity condition of Proposition \ref{prop_3} from the landscape visualization, implying that our method is more effective than the uniform curriculum.
Uncertainty is low for easily reachable goals (yellow in the success rate landscape) as well as barely reachable goals (purple). On the other hand, uncertainty of goals that are moderately reachable (green) is high, which indicates that the value uncertainty focuses more on informative goals and results in better performance as we see in Figure \ref{fig:quantitative}.


\subsection{Extensive Exploration for State Coverage}
We next evaluate the effectiveness of our method by qualitatively comparing the speed of state coverage of each method in the \textit{PointMazeSquareLarge} environment. Figure \ref{fig:density} demonstrates that VUVC efficiently increases the visited state entropy by considering the value uncertainty. Furthermore, after a sufficient number of exploration steps, the curriculum goal distribution induced by VUVC approaches a uniform distribution as the value uncertainty for every state converges to a consistent value. The results for other tasks are presented in Appendix \ref{D_3_Add_Res}.
\begin{figure}
    \centering
    \includegraphics[width=0.8\linewidth]{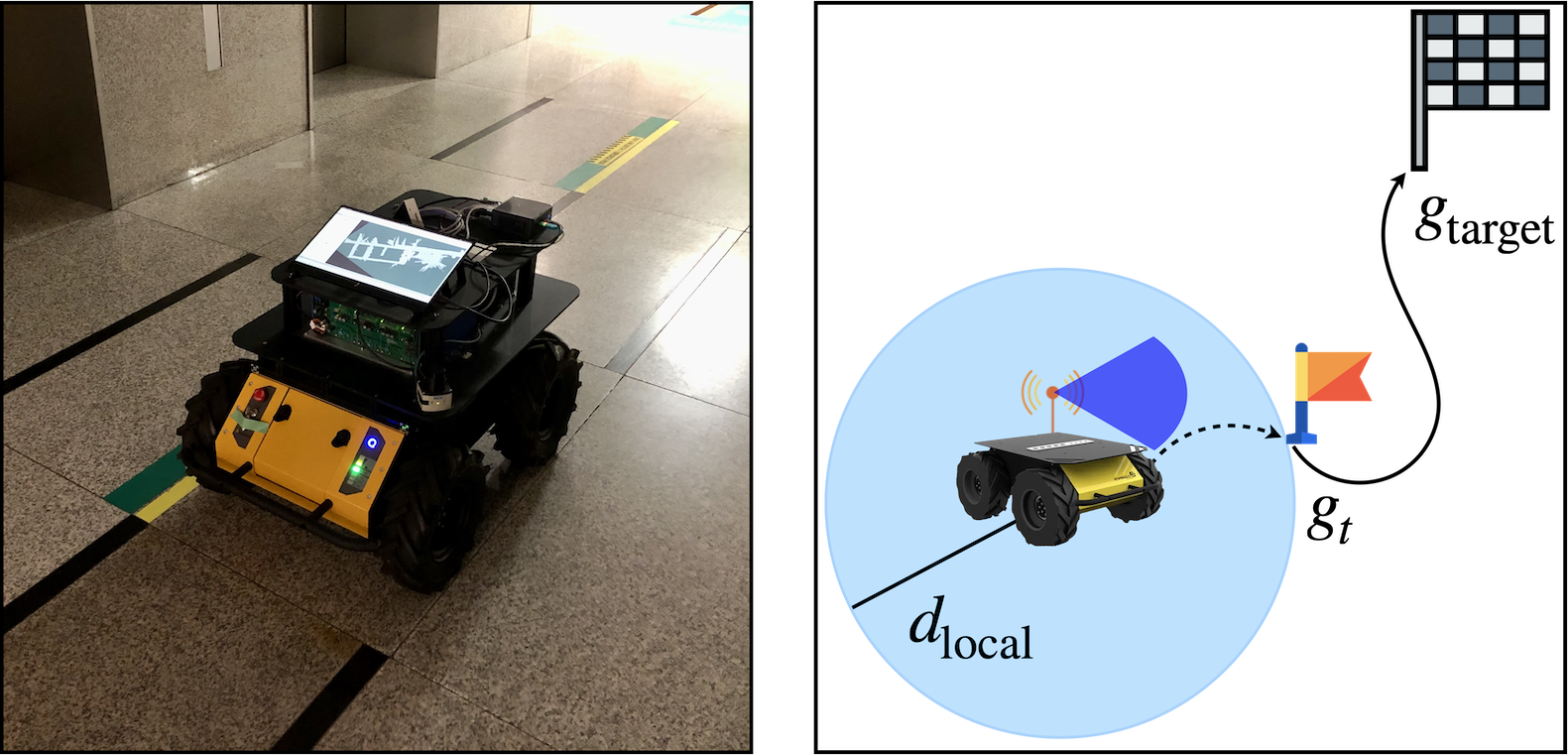}
    \caption{An illustrative example of how we utilize a global planner to generate a subgoal for our real robot platform.}
    \label{fig:real_exp_setup}
\end{figure}
\begin{figure}
    \centering
    \includegraphics[width=0.95\linewidth]{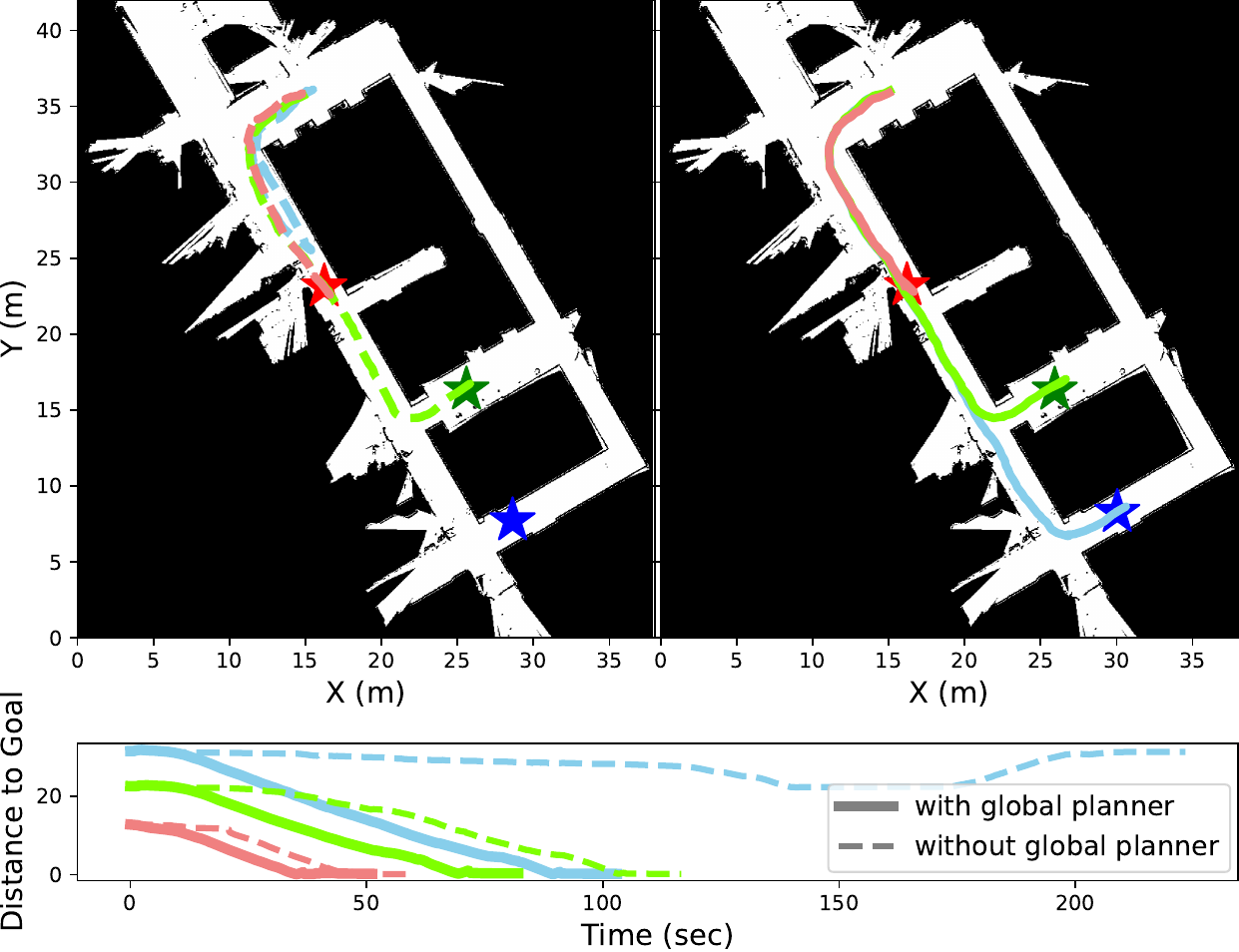}
    \caption{Building-scale navigation task with a real-world robot without (top left) and with global planner (top right). (Bottom) Evaluation on reaching the target goal.}
    \label{fig:real_exp_result}
    \vspace{-0.3cm}
\end{figure}




\subsection{Deploying Skills on the Real-world Robot}






We evaluate our method in a building-scale navigation task on the Husky A200 mobile robot which detects obstacles using a LiDAR sensor. We first apply our algorithm in a 2D navigation environment, and deploy learned navigation skills directly on the real robot in a zero-shot setup.
Figures \ref{fig:real_exp_setup} and \ref{fig:real_exp_result} show that the learned navigation skill can be directly used on our real-world robot without a manual design of complex reward functions and curriculum. We further demonstrate that combining learned navigation skills with the help of a global planner improves navigation performance. The learned skill aims to reach the local goal $g_t$ that is $d_{\mathrm{local}}$ away from the robot on the trajectory generated by the global planner. Figure \ref{fig:real_exp_result} demonstrates that the learned skill combined with the global planner reaches the goal faster (solid line) than the learned skill itself (dashed line).
Detailed description of the real-world experiment setup can be found in Appendix \ref{sec:real_exp_setup}.



%


\section{Related Work}

\subsection{Curriculum RL}

In GCRL, a goal relabeling scheme which samples goals from failed trajectories is proposed as an implicit curriculum method \cite{andrychowicz2017hindsight, fang2018dher, liu2018competitive, ding2019goal, fang2019curriculum, nair2018visual}. Another line of work investigates curriculum generation methods that consider task difficulty. These methods explicitly model a curriculum generative model, generating goals based on task difficulty \cite{florensa2018automatic, racaniere2019automated}, competence progress \cite{fournier2018accuracy}, utilization of an additional agent \cite{narvekar2019learning}, maximization of achieved goal distribution entropy with heuristic \cite{pitis2020maximum}, or progressive updating towards a predefined target distribution \cite{klink2020self}. However, prior works do not provide theoretical justification \cite{florensa2018automatic, racaniere2019automated}, are limited to a given target distribution \cite{fournier2018accuracy, narvekar2019learning, klink2020self}, or depend on manually engineered heuristics \cite{pitis2020maximum}. The notion of uncertainty has been also considered in VDS \cite{zhang2020automatic} which measures the uncertainty of the Q-functions to sample curriculum goals.
However, this work lacks theoretical justification and assumes an oracle goal sampler accessing a uniform distribution over all valid states in a state space, which artificially ignores exploration problems by resetting the agent to any state in the environment, whereas our work does not require such an assumption.





\subsection{Empowerment and Unsupervised Skill Learning}
Recent studies on empowerment have studied the forms of mutual information-based objectives to learn state-covering skills \cite{campos2020explore, pong2020skew}, promote skill diversity \cite{achiam2018variational, eysenbach2018diversity, liu2022learn}, learn non-parametric reward functions \cite{warde-farley2019unsupervised}, establish meta-training task distributions \cite{jabri2019unsupervised}, incorporate skill-transition dynamics models along with skill-conditioned policies for a model-based planning \cite{sharma2019dynamics}, and enhance generalization through the successor feature framework \cite{hansen2020fast, liu2021aps}. In addition, a number of works have studied how to extend empowerment to high-dimensional image space by using a non-parametric nearest neighbor to estimate entropy \cite{liu2021behavior, yarats2021reinforcement, seo2021state}. However, most of this research assumes a fixed stationary distribution over skills (or goals) and there has been little exploration regarding the form of skill (or goal) distribution $p(z)$ (or $p(g)$). Compared to prior empowerment approaches, we investigate the effectiveness of curriculum skill distribution.

\subsection{Uncertainty Quantification in RL}

Measures of uncertainty have played a key role in RL. Bootstrapped DQN \cite{osband2016deep} uses a bootstrapping method to estimate the uncertainty of the Q-value, and utilizes it for efficient exploration. Plan2Explore \cite{sekar2020planning} leverages an ensemble of one-step predictive models to guide the exploration. Both bootstrapping and dropout methods are used to measure the uncertainty of the collision prediction model for safe navigation \cite{kahn2017uncertainty}. PBP-RNN \cite{benatan2019fully} uses probabilistic backpropagation as an alternative to quantify uncertainty within a safe RL scenario. PETS \cite{chua2018deep} employs trajectory sampling with probabilistic dynamics models to bridge gap model-based RL and model-free RL.

% ...

\subsection{Intrinsic Reward and Exploration}
In a tabular setting, visit counts can be used as exploration bonus to encourage exploration \cite{strehl2008analysis}. Count-based exploration methods are further extended to non-tabular setting by introducing the pseudo-count \cite{bellemare2016unifying, ostrovski2017count} or successor representation \cite{machado2020count}. Another common approach guides the agent based on prediction errors. For instance, squared prediction error in learned dynamics models is used as exploration bonus \cite{stadie2015incentivizing}. RND \cite{burda2018exploration} uses errors in a randomly generated prediction problem that predicts the output of a fixed randomly initialized neural network given the observations.
Our work enables agents to reach any previously visited states by learning goal-conditioned policies that cover the entire goal space. In contrast, exploration bonuses help agents visit novel states, but they cannot reuse learned policies to solve user-specified goals as those states are quickly forgotten.





\section{Conclusion}

We provide the unifying framework VCRL which recasts MI-based RL as curriculum learning in goal-conditioned RL. Under VCRL framework, we propose a novel approach VUVC for unsupervised discovery of skills which utilizes a value uncertainty for an increment in the entropy of the visited state distribution. Under regularity conditions, we prove that VUVC improves the expected entropy more than the uniform curriculum method. Our experimental results demonstrate that VUVC consistently outperforms a variety of prior methods both on configuration-based and vision-based continuous robot manipulation tasks.
We also demonstrate that VUVC enables a real-world robot to learn to navigate in a long-range environment without any explicit rewards, and that incorporating skills with a global planner further improves the performance.

\section*{Acknowledgements}

This work was supported by the Industry Core Technology Development Project, 20005062, Development of Artificial Intelligence Robot Autonomous Navigation Technology for Agile Movement in Crowded Space, funded by the Ministry of Trade, Industry \& Energy (MOTIE, Republic of Korea) and by Institute of Information \& communications Technology Planning \& Evaluation (IITP) grant funded by the Korea government (MSIT) (No. 2022-0-00984, Development of Artificial Intelligence Technology for Personalized Plug-and-Play Explanation and Verification of Explanation, No.2019-0-00075, Artificial Intelligence Graduate School Program (KAIST)).




\bibliography{references}

\begin{thebibliography}{62}
\providecommand{\natexlab}[1]{#1}
\providecommand{\url}[1]{\texttt{#1}}
\expandafter\ifx\csname urlstyle\endcsname\relax
  \providecommand{\doi}[1]{doi: #1}\else
  \providecommand{\doi}{doi: \begingroup \urlstyle{rm}\Url}\fi

\bibitem[Achiam et~al.(2018)Achiam, Edwards, Amodei, and Abbeel]{achiam2018variational}
Achiam, J., Edwards, H., Amodei, D., and Abbeel, P.
\newblock Variational option discovery algorithms.
\newblock \emph{arXiv preprint arXiv:1807.10299}, 2018.

\bibitem[Andrychowicz et~al.(2017)Andrychowicz, Wolski, Ray, Schneider, Fong, Welinder, McGrew, Tobin, Pieter~Abbeel, and Zaremba]{andrychowicz2017hindsight}
Andrychowicz, M., Wolski, F., Ray, A., Schneider, J., Fong, R., Welinder, P., McGrew, B., Tobin, J., Pieter~Abbeel, O., and Zaremba, W.
\newblock Hindsight experience replay.
\newblock \emph{Advances in Neural Information Processing Systems (NeurIPS)}, 30, 2017.

\bibitem[Bellemare et~al.(2016)Bellemare, Srinivasan, Ostrovski, Schaul, Saxton, and Munos]{bellemare2016unifying}
Bellemare, M., Srinivasan, S., Ostrovski, G., Schaul, T., Saxton, D., and Munos, R.
\newblock Unifying count-based exploration and intrinsic motivation.
\newblock \emph{Advances in Neural Information Processing Systems (NeurIPS)}, 29, 2016.

\bibitem[Benatan \& Pyzer-Knapp(2019)Benatan and Pyzer-Knapp]{benatan2019fully}
Benatan, M. and Pyzer-Knapp, E.~O.
\newblock Fully bayesian recurrent neural networks for safe reinforcement learning.
\newblock \emph{arXiv preprint arXiv:1911.03308}, 2019.

\bibitem[Burda et~al.(2019)Burda, Edwards, Storkey, and Klimov]{burda2018exploration}
Burda, Y., Edwards, H., Storkey, A., and Klimov, O.
\newblock Exploration by random network distillation.
\newblock In \emph{International Conference on Learning Representations (ICLR)}, 2019.

\bibitem[Campos et~al.(2020)Campos, Trott, Xiong, Socher, Gir{\'o}-i Nieto, and Torres]{campos2020explore}
Campos, V., Trott, A., Xiong, C., Socher, R., Gir{\'o}-i Nieto, X., and Torres, J.
\newblock Explore, discover and learn: Unsupervised discovery of state-covering skills.
\newblock In \emph{International Conference on Machine Learning (ICML)}, pp.\  1317--1327. PMLR, 2020.

\bibitem[Choi et~al.(2021)Choi, Sharma, Lee, Levine, and Gu]{choi2021variational}
Choi, J., Sharma, A., Lee, H., Levine, S., and Gu, S.~S.
\newblock Variational empowerment as representation learning for goal-conditioned reinforcement learning.
\newblock In \emph{International Conference on Machine Learning (ICML)}, pp.\  1953--1963. PMLR, 2021.

\bibitem[Chua et~al.(2018)Chua, Calandra, McAllister, and Levine]{chua2018deep}
Chua, K., Calandra, R., McAllister, R., and Levine, S.
\newblock Deep reinforcement learning in a handful of trials using probabilistic dynamics models.
\newblock \emph{Advances in Neural Information Processing Systems (NeurIPS)}, 31, 2018.

\bibitem[Co-Reyes et~al.(2018)Co-Reyes, Liu, Gupta, Eysenbach, Abbeel, and Levine]{co2018self}
Co-Reyes, J., Liu, Y., Gupta, A., Eysenbach, B., Abbeel, P., and Levine, S.
\newblock Self-consistent trajectory autoencoder: Hierarchical reinforcement learning with trajectory embeddings.
\newblock In \emph{International Conference on Machine Learning (ICML)}, pp.\  1009--1018. PMLR, 2018.

\bibitem[Ding et~al.(2019)Ding, Florensa, Abbeel, and Phielipp]{ding2019goal}
Ding, Y., Florensa, C., Abbeel, P., and Phielipp, M.
\newblock Goal-conditioned imitation learning.
\newblock \emph{Advances in Neural Information Processing Systems (NeurIPS)}, 32, 2019.

\bibitem[Ecoffet et~al.(2019)Ecoffet, Huizinga, Lehman, Stanley, and Clune]{ecoffet2019go}
Ecoffet, A., Huizinga, J., Lehman, J., Stanley, K.~O., and Clune, J.
\newblock Go-explore: a new approach for hard-exploration problems.
\newblock \emph{arXiv preprint arXiv:1901.10995}, 2019.

\bibitem[Eysenbach et~al.(2019)Eysenbach, Gupta, Ibarz, and Levine]{eysenbach2018diversity}
Eysenbach, B., Gupta, A., Ibarz, J., and Levine, S.
\newblock Diversity is all you need: Learning skills without a reward function.
\newblock In \emph{International Conference on Learning Representations (ICLR)}, 2019.

\bibitem[Fang et~al.(2018)Fang, Zhou, Shi, Gong, Xu, and Zhang]{fang2018dher}
Fang, M., Zhou, C., Shi, B., Gong, B., Xu, J., and Zhang, T.
\newblock Dher: Hindsight experience replay for dynamic goals.
\newblock In \emph{International Conference on Learning Representations (ICLR)}, 2018.

\bibitem[Fang et~al.(2019)Fang, Zhou, Du, Han, and Zhang]{fang2019curriculum}
Fang, M., Zhou, T., Du, Y., Han, L., and Zhang, Z.
\newblock Curriculum-guided hindsight experience replay.
\newblock \emph{Advances in Neural Information Processing Systems (NeurIPS)}, 32, 2019.

\bibitem[Florensa et~al.(2018)Florensa, Held, Geng, and Abbeel]{florensa2018automatic}
Florensa, C., Held, D., Geng, X., and Abbeel, P.
\newblock Automatic goal generation for reinforcement learning agents.
\newblock In \emph{International Conference on Machine Learning (ICML)}, pp.\  1515--1528. PMLR, 2018.

\bibitem[Fournier et~al.(2018)Fournier, Sigaud, Chetouani, and Oudeyer]{fournier2018accuracy}
Fournier, P., Sigaud, O., Chetouani, M., and Oudeyer, P.-Y.
\newblock Accuracy-based curriculum learning in deep reinforcement learning.
\newblock \emph{arXiv preprint arXiv:1806.09614}, 2018.

\bibitem[Friston et~al.(2016)Friston, FitzGerald, Rigoli, Schwartenbeck, Pezzulo, et~al.]{friston2016active}
Friston, K., FitzGerald, T., Rigoli, F., Schwartenbeck, P., Pezzulo, G., et~al.
\newblock Active inference and learning.
\newblock \emph{Neuroscience \& Biobehavioral Reviews}, 68:\penalty0 862--879, 2016.

\bibitem[Friston et~al.(2021)Friston, Moran, Nagai, Taniguchi, Gomi, and Tenenbaum]{friston2021world}
Friston, K., Moran, R.~J., Nagai, Y., Taniguchi, T., Gomi, H., and Tenenbaum, J.
\newblock World model learning and inference.
\newblock \emph{Neural Networks}, 2021.

\bibitem[Gregor et~al.(2016)Gregor, Rezende, and Wierstra]{gregor2016variational}
Gregor, K., Rezende, D.~J., and Wierstra, D.
\newblock Variational intrinsic control.
\newblock \emph{arXiv preprint arXiv:1611.07507}, 2016.

\bibitem[Haarnoja et~al.(2018)Haarnoja, Zhou, Abbeel, and Levine]{haarnoja2018soft}
Haarnoja, T., Zhou, A., Abbeel, P., and Levine, S.
\newblock Soft actor-critic: Off-policy maximum entropy deep reinforcement learning with a stochastic actor.
\newblock In \emph{International Conference on Machine Learning (ICML)}, pp.\  1861--1870. PMLR, 2018.

\bibitem[Hansen et~al.(2020)Hansen, Dabney, Barreto, Warde-Farley, de~Wiele, and Mnih]{hansen2020fast}
Hansen, S., Dabney, W., Barreto, A., Warde-Farley, D., de~Wiele, T.~V., and Mnih, V.
\newblock Fast task inference with variational intrinsic successor features.
\newblock In \emph{International Conference on Learning Representations (ICLR)}, 2020.

\bibitem[Higgins et~al.(2017)Higgins, Matthey, Pal, Burgess, Glorot, Botvinick, Mohamed, and Lerchner]{higgins2017betavae}
Higgins, I., Matthey, L., Pal, A., Burgess, C., Glorot, X., Botvinick, M., Mohamed, S., and Lerchner, A.
\newblock beta-{VAE}: Learning basic visual concepts with a constrained variational framework.
\newblock In \emph{International Conference on Learning Representations (ICLR)}, 2017.

\bibitem[Islam et~al.(2019)Islam, Ahmed, and Precup]{islam2019marginalized}
Islam, R., Ahmed, Z., and Precup, D.
\newblock Marginalized state distribution entropy regularization in policy optimization.
\newblock \emph{arXiv preprint arXiv:1912.05128}, 2019.

\bibitem[Jabri et~al.(2019)Jabri, Hsu, Gupta, Eysenbach, Levine, and Finn]{jabri2019unsupervised}
Jabri, A., Hsu, K., Gupta, A., Eysenbach, B., Levine, S., and Finn, C.
\newblock Unsupervised curricula for visual meta-reinforcement learning.
\newblock \emph{Advances in Neural Information Processing Systems (NeurIPS)}, 32, 2019.

\bibitem[Kaelbling(1993)]{kaelbling1993learning}
Kaelbling, L.~P.
\newblock Learning to achieve goals.
\newblock In \emph{Proceedings of the International Joint Conference on Artificial Intelligence (IJCAI)}, volume~2, pp.\  1094--8. Citeseer, 1993.

\bibitem[Kahn et~al.(2017)Kahn, Villaflor, Pong, Abbeel, and Levine]{kahn2017uncertainty}
Kahn, G., Villaflor, A., Pong, V., Abbeel, P., and Levine, S.
\newblock Uncertainty-aware reinforcement learning for collision avoidance.
\newblock \emph{arXiv preprint arXiv:1702.01182}, 2017.

\bibitem[Kastner et~al.(2022)Kastner, Cox, Buiyan, and Lambrecht]{kastner2022all}
Kastner, L., Cox, J., Buiyan, T., and Lambrecht, J.
\newblock All-in-one: A drl-based control switch combining state-of-the-art navigation planners.
\newblock In \emph{International Conference on Robotics and Automation (ICRA)}, pp.\  2861--2867. IEEE, 2022.

\bibitem[Klink et~al.(2020)Klink, D'Eramo, Peters, and Pajarinen]{klink2020self}
Klink, P., D'Eramo, C., Peters, J.~R., and Pajarinen, J.
\newblock Self-paced deep reinforcement learning.
\newblock \emph{Advances in Neural Information Processing Systems (NeurIPS)}, 33, 2020.

\bibitem[Klyubin et~al.(2005)Klyubin, Polani, and Nehaniv]{klyubin2005empowerment}
Klyubin, A.~S., Polani, D., and Nehaniv, C.~L.
\newblock Empowerment: A universal agent-centric measure of control.
\newblock In \emph{IEEE Congress on Evolutionary Computation}, volume~1, pp.\  128--135. IEEE, 2005.

\bibitem[Lakshminarayanan et~al.(2017)Lakshminarayanan, Pritzel, and Blundell]{lakshminarayanan2017simple}
Lakshminarayanan, B., Pritzel, A., and Blundell, C.
\newblock Simple and scalable predictive uncertainty estimation using deep ensembles.
\newblock \emph{Advances in Neural Information Processing Systems (NeurIPS)}, 30, 2017.

\bibitem[Lee et~al.(2019)Lee, Eysenbach, Parisotto, Xing, Levine, and Salakhutdinov]{lee2019efficient}
Lee, L., Eysenbach, B., Parisotto, E., Xing, E., Levine, S., and Salakhutdinov, R.
\newblock Efficient exploration via state marginal matching.
\newblock \emph{arXiv preprint arXiv:1906.05274}, 2019.

\bibitem[Liu \& Abbeel(2021{\natexlab{a}})Liu and Abbeel]{liu2021aps}
Liu, H. and Abbeel, P.
\newblock Aps: Active pretraining with successor features.
\newblock In \emph{International Conference on Machine Learning (ICML)}, pp.\  6736--6747. PMLR, 2021{\natexlab{a}}.

\bibitem[Liu \& Abbeel(2021{\natexlab{b}})Liu and Abbeel]{liu2021behavior}
Liu, H. and Abbeel, P.
\newblock Behavior from the void: Unsupervised active pre-training.
\newblock \emph{Advances in Neural Information Processing Systems (NeurIPS)}, 34:\penalty0 18459--18473, 2021{\natexlab{b}}.

\bibitem[Liu et~al.(2018)Liu, Trott, Socher, and Xiong]{liu2018competitive}
Liu, H., Trott, A., Socher, R., and Xiong, C.
\newblock Competitive experience replay.
\newblock In \emph{International Conference on Learning Representations (ICLR)}, 2018.

\bibitem[Liu et~al.(2022)Liu, Wang, Tian, and Chen]{liu2022learn}
Liu, J., Wang, D., Tian, Q., and Chen, Z.
\newblock Learn goal-conditioned policy with intrinsic motivation for deep reinforcement learning.
\newblock In \emph{Proceedings of the AAAI Conference on Artificial Intelligence}, volume~36, pp.\  7558--7566, 2022.

\bibitem[Machado et~al.(2020)Machado, Bellemare, and Bowling]{machado2020count}
Machado, M.~C., Bellemare, M.~G., and Bowling, M.
\newblock Count-based exploration with the successor representation.
\newblock In \emph{Proceedings of the AAAI Conference on Artificial Intelligence}, volume~34, pp.\  5125--5133, 2020.

\bibitem[Mao et~al.(2017)Mao, Li, Xie, Lau, Wang, and Paul~Smolley]{mao2017least}
Mao, X., Li, Q., Xie, H., Lau, R.~Y., Wang, Z., and Paul~Smolley, S.
\newblock Least squares generative adversarial networks.
\newblock In \emph{Proceedings of the IEEE International Conference on Computer Vision (CVPR)}, pp.\  2794--2802, 2017.

\bibitem[Marsiglietti \& Kostina(2018)Marsiglietti and Kostina]{marsiglietti2018lower}
Marsiglietti, A. and Kostina, V.
\newblock A lower bound on the differential entropy of log-concave random vectors with applications.
\newblock \emph{Entropy}, 20\penalty0 (3):\penalty0 185, 2018.

\bibitem[Mendonca et~al.(2021)Mendonca, Rybkin, Daniilidis, Hafner, and Pathak]{mendonca2021discovering}
Mendonca, R., Rybkin, O., Daniilidis, K., Hafner, D., and Pathak, D.
\newblock Discovering and achieving goals via world models.
\newblock \emph{Advances in Neural Information Processing Systems (NeurIPS)}, 34:\penalty0 24379--24391, 2021.

\bibitem[Nair et~al.(2018)Nair, Pong, Dalal, Bahl, Lin, and Levine]{nair2018visual}
Nair, A.~V., Pong, V., Dalal, M., Bahl, S., Lin, S., and Levine, S.
\newblock Visual reinforcement learning with imagined goals.
\newblock \emph{Advances in Neural Information Processing Systems (NeurIPS)}, 31, 2018.

\bibitem[Narvekar \& Stone(2019)Narvekar and Stone]{narvekar2019learning}
Narvekar, S. and Stone, P.
\newblock Learning curriculum policies for reinforcement learning.
\newblock In \emph{Proceedings of the 18th International Conference on Autonomous Agents and MultiAgent Systems (AAMAS)}, pp.\  25--33, 2019.

\bibitem[Osband et~al.(2016)Osband, Blundell, Pritzel, and Van~Roy]{osband2016deep}
Osband, I., Blundell, C., Pritzel, A., and Van~Roy, B.
\newblock Deep exploration via bootstrapped dqn.
\newblock \emph{Advances in Neural Information Processing Systems (NeurIPS)}, 29, 2016.

\bibitem[Osband et~al.(2018)Osband, Aslanides, and Cassirer]{osband2018randomized}
Osband, I., Aslanides, J., and Cassirer, A.
\newblock Randomized prior functions for deep reinforcement learning.
\newblock \emph{Advances in Neural Information Processing Systems (NeurIPS)}, 31, 2018.

\bibitem[Ostrovski et~al.(2017)Ostrovski, Bellemare, Oord, and Munos]{ostrovski2017count}
Ostrovski, G., Bellemare, M.~G., Oord, A., and Munos, R.
\newblock Count-based exploration with neural density models.
\newblock In \emph{International Conference on Machine Learning (ICML)}, pp.\  2721--2730. PMLR, 2017.

\bibitem[Parr et~al.(2022)Parr, Pezzulo, and Friston]{parr2022active}
Parr, T., Pezzulo, G., and Friston, K.~J.
\newblock Active inference: the free energy principle in mind, brain, and behavior.
\newblock \emph{MIT Press}, 2022.

\bibitem[Pitis et~al.(2020)Pitis, Chan, Zhao, Stadie, and Ba]{pitis2020maximum}
Pitis, S., Chan, H., Zhao, S., Stadie, B., and Ba, J.
\newblock Maximum entropy gain exploration for long horizon multi-goal reinforcement learning.
\newblock In \emph{International Conference on Machine Learning (ICML)}, pp.\  7750--7761. PMLR, 2020.

\bibitem[Plappert et~al.(2018)Plappert, Andrychowicz, Ray, McGrew, Baker, Powell, Schneider, Tobin, Chociej, Welinder, et~al.]{plappert2018multi}
Plappert, M., Andrychowicz, M., Ray, A., McGrew, B., Baker, B., Powell, G., Schneider, J., Tobin, J., Chociej, M., Welinder, P., et~al.
\newblock Multi-goal reinforcement learning: Challenging robotics environments and request for research.
\newblock \emph{arXiv preprint arXiv:1802.09464}, 2018.

\bibitem[Pong et~al.(2019)Pong, Dalal, Lin, and Nair]{pong2019rlkit}
Pong, V., Dalal, M., Lin, S., and Nair, A.
\newblock Rlkit.
\newblock \emph{URL: https://github. com/vitchyr/rlkit}, 2019.

\bibitem[Pong et~al.(2020)Pong, Dalal, Lin, Nair, Bahl, and Levine]{pong2020skew}
Pong, V., Dalal, M., Lin, S., Nair, A., Bahl, S., and Levine, S.
\newblock Skew-fit: State-covering self-supervised reinforcement learning.
\newblock In \emph{International Conference on Machine Learning (ICML)}, pp.\  7783--7792. PMLR, 2020.

\bibitem[Racaniere et~al.(2019)Racaniere, Lampinen, Santoro, Reichert, Firoiu, and Lillicrap]{racaniere2019automated}
Racaniere, S., Lampinen, A., Santoro, A., Reichert, D., Firoiu, V., and Lillicrap, T.
\newblock Automated curriculum generation through setter-solver interactions.
\newblock In \emph{International Conference on Learning Representations (ICLR)}, 2019.

\bibitem[Ren et~al.(2019)Ren, Dong, Zhou, Liu, and Peng]{ren2019exploration}
Ren, Z., Dong, K., Zhou, Y., Liu, Q., and Peng, J.
\newblock Exploration via hindsight goal generation.
\newblock \emph{Advances in Neural Information Processing Systems (NeurIPS)}, 32, 2019.

\bibitem[Salge et~al.(2014)Salge, Glackin, and Polani]{salge2014empowerment}
Salge, C., Glackin, C., and Polani, D.
\newblock Empowerment--an introduction.
\newblock In \emph{Guided Self-Organization: Inception}, pp.\  67--114. Springer, 2014.

\bibitem[Schaul et~al.(2015)Schaul, Horgan, Gregor, and Silver]{schaul2015universal}
Schaul, T., Horgan, D., Gregor, K., and Silver, D.
\newblock Universal value function approximators.
\newblock In \emph{International Conference on Machine Learning (ICML)}, pp.\  1312--1320. PMLR, 2015.

\bibitem[Sekar et~al.(2020)Sekar, Rybkin, Daniilidis, Abbeel, Hafner, and Pathak]{sekar2020planning}
Sekar, R., Rybkin, O., Daniilidis, K., Abbeel, P., Hafner, D., and Pathak, D.
\newblock Planning to explore via self-supervised world models.
\newblock In \emph{International Conference on Machine Learning (ICML)}, pp.\  8583--8592. PMLR, 2020.

\bibitem[Seo et~al.(2021)Seo, Chen, Shin, Lee, Abbeel, and Lee]{seo2021state}
Seo, Y., Chen, L., Shin, J., Lee, H., Abbeel, P., and Lee, K.
\newblock State entropy maximization with random encoders for efficient exploration.
\newblock In \emph{International Conference on Machine Learning (ICML)}, pp.\  9443--9454. PMLR, 2021.

\bibitem[Sharma et~al.(2019)Sharma, Gu, Levine, Kumar, and Hausman]{sharma2019dynamics}
Sharma, A., Gu, S., Levine, S., Kumar, V., and Hausman, K.
\newblock Dynamics-aware unsupervised discovery of skills.
\newblock In \emph{International Conference on Learning Representations (ICLR)}, 2019.

\bibitem[Stadie et~al.(2015)Stadie, Levine, and Abbeel]{stadie2015incentivizing}
Stadie, B.~C., Levine, S., and Abbeel, P.
\newblock Incentivizing exploration in reinforcement learning with deep predictive models.
\newblock \emph{arXiv preprint arXiv:1507.00814}, 2015.

\bibitem[Strehl \& Littman(2008)Strehl and Littman]{strehl2008analysis}
Strehl, A.~L. and Littman, M.~L.
\newblock An analysis of model-based interval estimation for markov decision processes.
\newblock \emph{Journal of Computer and System Sciences}, 74\penalty0 (8):\penalty0 1309--1331, 2008.

\bibitem[Trott et~al.(2019)Trott, Zheng, Xiong, and Socher]{trott2019keeping}
Trott, A., Zheng, S., Xiong, C., and Socher, R.
\newblock Keeping your distance: Solving sparse reward tasks using self-balancing shaped rewards.
\newblock \emph{Advances in Neural Information Processing Systems (NeurIPS)}, 32, 2019.

\bibitem[Warde-Farley et~al.(2019)Warde-Farley, de~Wiele, Kulkarni, Ionescu, Hansen, and Mnih]{warde-farley2019unsupervised}
Warde-Farley, D., de~Wiele, T.~V., Kulkarni, T., Ionescu, C., Hansen, S., and Mnih, V.
\newblock Unsupervised control through non-parametric discriminative rewards.
\newblock In \emph{International Conference on Learning Representations (ICLR)}, 2019.

\bibitem[Yarats et~al.(2021)Yarats, Fergus, Lazaric, and Pinto]{yarats2021reinforcement}
Yarats, D., Fergus, R., Lazaric, A., and Pinto, L.
\newblock Reinforcement learning with prototypical representations.
\newblock In \emph{International Conference on Machine Learning (ICML)}, pp.\  11920--11931. PMLR, 2021.

\bibitem[Zhang et~al.(2020)Zhang, Abbeel, and Pinto]{zhang2020automatic}
Zhang, Y., Abbeel, P., and Pinto, L.
\newblock Automatic curriculum learning through value disagreement.
\newblock \emph{Advances in Neural Information Processing Systems (NeurIPS)}, 33, 2020.

\end{thebibliography}
\bibliographystyle{icml2023}

\newpage
\appendix
\onecolumn



\section{Limitations}\label{sec:limitations}

We summarize the limitations of our work as follows:
\begin{itemize}
    \item Although VUVC demonstrates significant improvements in both sample efficiency and the ability to cover the state space, the quantitative experimental results suggest that divergence during training is a potential problem, particularly in tasks such as \textit{FetchPickAndPlace}, which has a higher dimensionality of state space compared to others. It is likely that divergence occurs in this task when there are states in a confined space that exhibit high value uncertainty, causing VUVC to focus on sampling goals around these states for a certain period of time.
    \item Demonstrating that the regularity conditions of Proposition \ref{prop_2} and \ref{prop_3} hold is limited in empirical study (see Figure \ref{fig:prop}). Therefore, it is an appealing research direction to rigorously show the regularity conditions hold.
    \item In our experiment, we consider a fixed initial state. Even though the core concept of our approach, which estimates the uncertainty of the learned value functions, remains applicable to variable initial states, its performance might be affected negatively due to the increased training data required to handle a wide range of initial states. We have not yet validated the scalability of our approach in environments with non-fixed initial states, and leave it as future work.
\end{itemize}

\section{Relationship with Active Inference}\label{sec:additional_related_work}

Our work is also related to active inference \cite{friston2016active, friston2021world, parr2022active}. Active inference can play a crucial role in the context of world models, as it allows the agent to update its beliefs based on the actions performed by changing the gathered observations. For example, to efficiently learn a world model, Plan2Explore \cite{sekar2020planning} and LEXA \cite{mendonca2021discovering} agents seeks out surprising states by leveraging ensembles of world models to guide their exploration. This can be related to our approach, VUVC, which seek out goals that the agent learns the most from. Moreover, in our approach, we focus on maximizing state-marginal mutual information $I(s;z)$, but if we maximize state-predictive mutual information $I(s'; z|s)$, as in the DADS method \cite{sharma2019dynamics}, we would learn skill-transition dynamics models, which might be considered as world models. From an active inference perspective, this could lead the agent to select actions and collect observations in a manner that reduces the uncertainty associated with skill-transition dynamics. 

\section{Proofs}\label{A_Proofs}

\Firstprop*
\begin{proof}
The mutual information can be rewritten as the difference between conditional entropy and marginal entropy, which correspond to, respectively, the aleatoric uncertainty and predictive entropy:
\begin{align}
    \mathcal{I}(V_{\psi}(s_0, g);\psi| s_0, g) = \mathcal{H}(V_{\psi}(s_0, g)|s_0, g)-\mathcal{H}(V_{\psi}(s_0, g)|\psi,s_0,g).
\end{align}
When the value function is deterministic with zero variance, maximizing the mutual information is equal to maximizing the marginal entropy. As shown in \cite{marsiglietti2018lower}, a lower bound on the entropy of a log-concave random variable can be derived in terms of the $p$-th absolute moment: 
\begin{align}
    \mathcal{H}(V_{\psi}(s_0, g)|s_0,g) \ge \log \left( \frac{2\|V_{\psi}(s_0, g)-\mathbb{E}[V_{\psi}(s_0, g)]\|_p}{\Gamma(p+1)^{\frac{1}{p}}} \right),
\end{align}
where $\Gamma$ denotes the Gamma function. Moreover, for $p=2$, the bound tightens as   
\begin{align}
    \mathcal{H}(V_{\psi}(s_0, g)|s_0,g) \ge \log (2\sqrt{\text{Var}[V_{\psi}(s_0, g)]}),
\end{align}
which implies selecting a skill/goal that maximizes the disagreement in predictions of an ensemble of value functions is equivalent to maximizing the lower bound approximation of the mutual information.
\end{proof}

\Secondprop*
\begin{proof}
Given a visited state $s'$ at time $t+1$, the next step empirical distribution of the visited state can be written as  
\begin{align}
    p_{t+1}^{\mathrm{visited}}(s|s') = \frac{p_{t}^{\mathrm{visited}}(s) + \epsilon\mathbb{I}[s=s']}{1+\epsilon}.
\end{align}
With a curriculum goal $g$ and a stationary state distribution induced by the policy $\rho^{\pi_{\theta}}(s|g)$, a next step empirical distribution of the visited state can be written as 
\begin{align*}
    p_{t+1}^{\mathrm{visited}}(s) &= \sum_{s'} p_{t+1}^{\mathrm{visited}}(s|s') \rho^{\pi_{\theta}}(s'|g) \\
    &= \sum_{s'}\rho^{\pi_{\theta}}(s'|g)\left(  \frac{p_{t}^{\mathrm{visited}}(s) + \epsilon\mathbb{I}[s=s']}{1+\epsilon} \right) \\
    &= \frac{p_{t}^{\mathrm{visited}}(s) + \epsilon\rho^{\pi_{\theta}}(s|g)}{1+\epsilon}. \numberthis
\end{align*}
Substituting the expression of $p_{t+1}^{\mathrm{visited}}(s)$ into entropy increment over uniform curriculum gives
\begin{align*}
    I_t &= \mathbb{E}_{g\sim p_t^{\mathrm{VU}}}[\mathcal{H}(p_{t+1}^{\mathrm{visited}})] - \mathbb{E}_{g\sim p_t^{\mathcal{U}}}[\mathcal{H}(p_{t+1}^{\mathrm{visited}})] \\
    &= \sum_g (p_t^{\mathrm{VU}}(g) - p_t^{\mathcal{U}}(g))\sum_{s}-\frac{p_{t}^{\mathrm{visited}}(s) + \epsilon\rho^{\pi_{\theta}}(s|g)}{1 + \epsilon} \log \frac{p_{t}^{\mathrm{visited}}(s) + \epsilon\rho^{\pi_{\theta}}(s|g)}{1 + \epsilon}. \numberthis
\end{align*}
We take the derivative with respect to $\epsilon$ and consider the asymptotic behavior where $\epsilon \to 0$:
\begin{align*}
&\lim_{\epsilon \to 0}\frac{\partial}{\partial \epsilon}I_t \\
&\quad= \sum_g (p_t^{\mathrm{VU}}(g) - p_t^{\mathcal{U}}(g))\sum_{s}\lim_{\epsilon \to 0}\frac{\partial}{\partial \epsilon} -\frac{p_{t}^{\mathrm{visited}}(s) + \epsilon\rho^{\pi_{\theta}}(s|g)}{1 + \epsilon} \log \frac{p_{t}^{\mathrm{visited}}(s) + \epsilon\rho^{\pi_{\theta}}(s|g)}{1 + \epsilon} \\
&\quad= \sum_g (p_t^{\mathrm{VU}}(g) - p_t^{\mathcal{U}}(g))\sum_{s}\lim_{\epsilon \to 0}-\frac{1}{1 + \epsilon^2}\left( (\rho^{\pi_{\theta}}(s|g) - p_{t}^{\mathrm{visited}}(s))(\log\frac{p_{t}^{\mathrm{visited}}(s) + \epsilon\rho^{\pi_{\theta}}(s|g)}{1 + \epsilon} + 1)\right) \\ 
&\quad= \sum_g (p_t^{\mathrm{VU}}(g) - p_t^{\mathcal{U}}(g))\sum_{s}-(\rho^{\pi_{\theta}}(s|g) - p_{t}^{\mathrm{visited}}(s))(\log p_{t}^{\mathrm{visited}}(s) + 1) \\
&\quad= \sum_g (p_t^{\mathrm{VU}}(g) - p_t^{\mathcal{U}}(g))\sum_{s}\left( -\rho^{\pi_{\theta}}(s|g)\log p_{t}^{\mathrm{visited}}(s) + p_{t}^{\mathrm{visited}}(s)\log p_{t}^{\mathrm{visited}}(s) \right).
\end{align*}
Substituting $\rho^{\pi_{\theta}}(s|g)=\mathbb{I}(s=g)$ simplifies
\begin{align*}
\lim_{\epsilon \to 0}\frac{\partial}{\partial \epsilon}I_t &= \sum_g (p_t^{\mathrm{VU}}(g) - p_t^{\mathcal{U}}(g))\sum_{s} \left(-\mathbb{I}[s=g]\log p_{t}^{\mathrm{visited}}(s) + p_{t}^{\mathrm{visited}}(s)\log p_{t}^{\mathrm{visited}}(s) \right)\\
&= \sum_g (p_t^{\mathrm{VU}}(g) - p_t^{\mathcal{U}}(g))(-\log p_{t}^{\mathrm{visited}}(g) - \mathbb{E}[{-\log p_{t}^{\mathrm{visited}}(g)}]) \\
&= -\sum_g (p_t^{\mathrm{VU}}(g) - \mathbb{E}[{p_t^{\mathrm{VU}}(g)}])(\log p_{t}^{\mathrm{visited}}(g) - \mathbb{E}[{\log p_{t}^{\mathrm{visited}}(g)}]) \\
&\quad+\sum_g (p_t^{\mathcal{U}}(g) - \mathbb{E}[{p_t^{\mathcal{U}}(g)}])(\log p_{t}^{\mathrm{visited}}(g) - \mathbb{E}[{\log p_{t}^{\mathrm{visited}}(g)}]) \\
&= -\mathrm{Cov}[U(g),\log p_t^{\mathrm{visited}}(g)],
\end{align*}
where we use the fact that $p_t^{\mathrm{VU}}(g)=\frac{1}{Z_t}p_t^{\mathcal{U}}(g)U(g)$ and $p_t^{\mathcal{U}}(g)=\mathbb{E}[{p_t^{\mathcal{U}}(g)}]$ for all $g$. Thus we can complete the proof.
\end{proof}
\Thirdprop*

\begin{proof}
We substitute the $\mathcal{G}=\mathcal{G}_{\mathrm{exploit}} \cup \mathcal{G}_{\mathrm{uninfo}} \cup \mathcal{G}_{\mathrm{info}}$ into the entropy increment over uniform curriculum and expand the expression:
\begin{align*}
&\lim_{\epsilon \to 0}\frac{\partial}{\partial \epsilon}I_t \\
&\quad= \sum_{g \in \mathcal{G}_{\mathrm{exploit}}} (p_t^{\mathrm{VU}}(g) - p_t^{\mathcal{U}}(g))\sum_{s} \left(-\mathbb{I}[s=g]\log p_{t}^{\mathrm{visited}}(s) + p_{t}^{\mathrm{visited}}(s)\log p_{t}^{\mathrm{visited}}(s) \right) \\
&\quad\quad+ \sum_{g \in \mathcal{G}_{\mathrm{uninfo}}} (p_t^{\mathrm{VU}}(g) - p_t^{\mathcal{U}}(g))\sum_{s}\left( -\rho_{\mathrm{uninfo}}^{\pi_{\theta}}(s|g)\log p_{t}^{\mathrm{visited}}(s) + p_{t}^{\mathrm{visited}}(s)\log p_{t}^{\mathrm{visited}}(s) \right) \\
&\quad\quad+ \sum_{g \in \mathcal{G}_{\mathrm{info}}} (p_t^{\mathrm{VU}}(g) - p_t^{\mathcal{U}}(g))\sum_{s}\left( -\rho_{\mathrm{info}}^{\pi_{\theta}}(s|g)\log p_{t}^{\mathrm{visited}}(s) + p_{t}^{\mathrm{visited}}(s)\log p_{t}^{\mathrm{visited}}(s) \right) \\
&\quad= \sum_{g \in \mathcal{G}} (p_t^{\mathrm{VU}}(g) - p_t^{\mathcal{U}}(g))\sum_{s} \left(-\mathbb{I}[s=g]\log p_{t}^{\mathrm{visited}}(s) + p_{t}^{\mathrm{visited}}(s)\log p_{t}^{\mathrm{visited}}(s) \right) \\
&\quad\quad+ \sum_{g \in \mathcal{G}_{\mathrm{uninfo}}} (p_t^{\mathrm{VU}}(g) - p_t^{\mathcal{U}}(g)) \left( \log p_{t}^{\mathrm{visited}}(g) - \sum_{s} \rho^{\pi_{\theta}}(s|g)\log p_{t}^{\mathrm{visited}}(s) \right) \\
&\quad\quad+ \sum_{g \in \mathcal{G}_{\mathrm{info}}} (p_t^{\mathrm{VU}}(g) - p_t^{\mathcal{U}}(g)) \left( \log p_{t}^{\mathrm{visited}}(g) - \sum_{s} \rho^{\pi_{\theta}}(s|g)\log p_{t}^{\mathrm{visited}}(s) \right) \\
&\quad= -\mathrm{Cov}[U(g), \log p_t^{\mathrm{visited}}(g)] - \Delta_1 \cdot \sum_{g \in \mathcal{G}_{\mathrm{uninfo}}} (p_t^{\mathrm{VU}}(g) - p_t^{\mathcal{U}}(g)) + \Delta_2 \cdot \sum_{g \in \mathcal{G}_{\mathrm{info}}} (p_t^{\mathrm{VU}}(g) - p_t^{\mathcal{U}}(g))
\end{align*}
Following the assumptions, we can conclude the proof.
\end{proof}



\section{Experimental Setup Details}
\label{B_Exp_Details}

\subsection{Environments}
As we described, we adopt the point mazes from the prior works \cite{zhang2020automatic, trott2019keeping}. Following these works, an agent observes a position of a point and takes action given a 2-dimensional goal position. For configuration-based robotic manipulation tasks, we utilize \textit{Fetch} environments \cite{plappert2018multi} whose initial and goal distribution are modified to consider more complicated tasks, following the prior work \cite{ren2019exploration}. In these environments, an observation includes gripper position and velocity, gripper state, and object position and velocity. Given a 3-dimensional desired object position as a goal, it takes action to move its end-effector in Cartesian coordinates and open/close the gripper. For vision-based manipulation tasks, we adopt \textit{Sawyer} environments \cite{nair2018visual} which manipulate a 7-DoF Sawyer robotic arm solely from a visual input without any explicit positional information of either a robotic arm or an object. A task to solve is given as a desired goal image which the agent should match its observation image with. \textit{HuskyNavigate} is the environment in which we train navigation skills that we deployed on the real robot. An observation consists of the raw 2D laser measurements, the relative goal position, and current robot velocity. Given a 2-dimensional goal position, an action command which consists of linear velocity and angular velocity is given to the robot.
Details about the environments are summarized in Table \ref{tab:env_details}.

\begin{table*}[h]
    \centering
    \begin{tabular}{c|c|c|c|c|c|c}
        \toprule
        \textbf{Parameter} & \textbf{all \textit{PointMaze}} & \textbf{all \textit{Fetch}} & \textbf{\textit{SawyerDoorHook}} & \textbf{\textit{SawyerPickup}} & \textbf{\textit{SawyerPush}}  & \textbf{\textit{HuskyNavigate}} \\
        \midrule
        State space $\mathcal{S}$ & $\in \mathbb{R}^{2}$ & $\in \mathbb{R}^{25}$ & $\in \mathbb{R}^{48\times48\times3}$ & $\in \mathbb{R}^{48\times48\times3}$ & $\in \mathbb{R}^{48\times48\times3}$ & $\in \mathbb{R}^{364}$ \\
        Action space $\mathcal{A}$ & $\in \mathbb{R}^{2}$ & $\in \mathbb{R}^{4}$ & $\in \mathbb{R}^{3}$ & $\in \mathbb{R}^{3}$ & $\in \mathbb{R}^{2}$ & $\in \mathbb{R}^{2}$ \\
        Goal space $\mathcal{G}$ & $\in \mathbb{R}^{2}$ & $\in \mathbb{R}^{3}$ & $\in \mathbb{R}^{48\times48\times3}$ & $\in \mathbb{R}^{48\times48\times3}$ & $\in \mathbb{R}^{48\times48\times3}$ & $\in \mathbb{R}^{2}$ \\
        Episode length & 50 & 50 & 100 & 50 & 50 & 1000 \\
        \bottomrule
    \end{tabular}
    \hspace{1cm}
    \caption{Environment details for each experiment.}
    \label{tab:env_details}
\end{table*}

\subsection{Baseline Algorithms}
We evaluate sample efficiency and state coverage speed of VUVC compared to the following baseline methods:
\begin{itemize}
\item{\textbf{Hindsight Experience Replay (HER)} \cite{andrychowicz2017hindsight}}: 
HER is a na\"ive goal-conditioned RL method. The key idea of HER is to construct implicit curriculum goals by revisiting previous states in the experience replay. By storing additional trajectories using these curriculum goals, HER generates reward signals, even in situations where the initial sparse reward fails to provide meaningful feedback.
%
\item{\textbf{Reinforcement learning with Imagined Goals (RIG)} \cite{nair2018visual}}: 
RIG trains a goal-conditioned policy in an unsupervised manner by estimating the visited state distribution and automatically setting curriculum goals sampled from this distribution.

\item{\textbf{GoalGAN}} \cite{florensa2018automatic}:
GoalGAN encourages an agent to explore by suggesting curriculum goals from the generative model \cite{mao2017least}. To encourage an agent to explore the environment, it generates goals of intermediate difficulty where the difficulty of task (or goal) is measured from a success rate over some number of trials to solve the task. In vision-based robotic manipulation tasks, we compute the success rate in the latent space where an encoder is inherited from RIG which adopts a VAE as a state density estimate model.

\item{\textbf{Diversity Is All You Need (DIAYN)} \cite{eysenbach2018diversity}}: DIAYN learns a latent skill based on mutual information maximization between skills and visited states with policy entropy regularization. It also reduces the mutual information between actions and skills, given the state, in order to separate the skills from each other, and partitions the visited state space into separate sections for each skill, each of which has a uniform stationary prior distribution.
%
\item{\textbf{Explore, Discover and Learn (EDL)} \cite{campos2020explore}}: 
EDL overcomes the limitation of existing variational empowerment methods which provide a poor coverage of the state space.
Unlike RIG which makes use of the current goal-conditioned policy to approximate the state distribution, 
EDL utilizes a \textit{fixed} uniform distribution over all $\mathcal{S}$ which would require an oracle sampler from the set of valid states.
If the oracle is unavailable, an exploration policy is employed to induce the uniform distribution across valid states. The skill (or goal) distribution is inferred from this uniform state distribution by using a VAE. Then, the state-covering policy is trained based on the learned skill distribution following the variational empowerment objective.
%
\item{\textbf{Skew-Fit} \cite{pong2020skew}}: 
Skew-Fit aims to achieve a general-purpose policy that can accomplish new user-specified goals, in an unsupervised manner. To achieve such goal-conditioned policy, Skew-Fit estimates the visited state distribution like RIG, and skews this distribution with the negative exponent so that the skewed distribution converges to a uniform distribution over states. Under the assumption that the goal space is equivalent to the state space, goals are sampled from the skewed distribution when training, which implies that the goal distribution is non-stationary as the visited state space gets larger.
\end{itemize}

\section{Implementation Details and Hyperparameters}
\label{C_Imple_Details}


For all experiments, the agents are trained with SAC \cite{haarnoja2018soft} with an automatically tuned entropy coefficient. During the training of the RL agent, we relabel transitions with goals by sampling from the curriculum goal distribution with probability 0.5 and the future goals with probability 0.3. We use $\beta$-VAE for both modeling a state density and computing an intrinsic reward, $\log q_\lambda(g|s)$. For DIAYN, we use a fixed set of 100 skills. To evaluate DIAYN's goal-reaching performance, we estimate the target skill from the desired goal using a discriminator. This estimated skill is then used as the goal for goal-conditioned policy. For RIG and Skew-Fit in \textit{Sawyer} experiments, we use hyperparameters inherited from the official implementation of Skew-Fit \cite{pong2019rlkit}\footnote{\href{https://github.com/rail-berkeley/rlkit}{\url{github.com/rail-berkeley/rlkit}}}. For the \textit{PointMaze} and \textit{Fetch} experiments, we add exploration noise into the action after a goal is reached by following the Go-Explore \cite{ecoffet2019go} and normalize the observations using a running mean and standard deviation. For the \textit{Sawyer} experiment, we normalize the image observations to be in the interval $[0, 1]$ by dividing by the maximum pixel intensity. Normalization is especially crucial for training $\beta$-VAE on all environments.
For training an exploration policy in EDL, we use the entropy of the marginal state distribution \cite{co2018self, islam2019marginalized} as a reward to encourage the agent to visit less visited states more.

The training time on a single NVIDIA Quadro 8000 GPU can range from 6 to 30 hours depending on the task and the situation.





\begin{table*}[h]
    \centering
    \begin{tabular}{c|c}
        \toprule
        \textbf{Hyperparameter} & \textbf{Value} \\ 
        \midrule
        Discount factor & 0.98  \\
        Replay buffer size & 1000000 \\
        Episode length & 50  \\
        RL batch size & 2048   \\
        Observation normalization & \{\textbf{Yes}, No\}   \\
        Polyak averaging coefficient for target networks & \{0.001, \textbf{0.005}\} \\
        Policy hidden activation & ReLU  \\
        Policy learning rate & \{0.0003, \textbf{0.001}, 0.003\}  \\
        Q-Function hidden activation & ReLU  \\
        Q-Function learning rate & \{0.0003, \textbf{0.001}, 0.003\}  \\
        Ensemble size for quantifying value uncertainty & \{\textbf{3}, 5, 7\}  \\
        VAE batch size & 256  \\
        VAE latent dimension size & 2  \\ 
        VAE encoder activation & ReLU  \\
        VAE decoder activation & ReLU  \\
        VAE learning rate & \{0.0003, \textbf{0.001}, 0.003\}  \\
        $\beta$ for $\beta$-VAE & \{5, \textbf{10}, 20\} \\
        $\alpha$ for Skew & -1  \\
        \bottomrule
    \end{tabular}
    \hspace{1cm}
    \caption{General hyperparameters used for all \textit{PointMaze} and \textit{Fetch} experiments. Values between brackets are tuned independently using a grid search.}    
    \label{tab:hyparam_point_fetch}
\end{table*}

\begin{table*}[h]
    \centering
    \begin{tabular}{c|c|c|c}
        \toprule
        \textbf{Hyperparameter} & \textbf{all \textit{PointMaze}} & \textbf{\textit{FetchPush/PickAndPlace}} & \textbf{\textit{FetchSlide}} \\ 
        \midrule
        Minimum \# steps in replay buffer before training & 5000 & \{5000, \textbf{20000}, 50000\} &  \{5000, 20000, \textbf{50000}\} \\
        \bottomrule
    \end{tabular}
    \hspace{1cm}
    \caption{Specific hyperparameters for all \textit{PointMaze} and \textit{Fetch} experiments. Values between brackets are tuned independently using a grid search.} 
    \label{tab:hyparam_2}
\end{table*}


\begin{table*}[h]
    \centering
    \begin{tabular}{c|c}
        \toprule
        \textbf{Hyperparameter} & \textbf{Value} \\ 
        \midrule
        Discount factor & 0.99   \\
        Replay buffer size & 100000 \\
        RL batch size & 1024   \\
        Policy hidden activation & ReLU  \\
        Policy learning rate & 0.001  \\
        Q-Function hidden activation & ReLU  \\
        Q-Function learning rate & 0.001  \\
        Ensemble size for quantifying value uncertainty & 3  \\
        VAE batch size & 64  \\
        \bottomrule
    \end{tabular}
    \hspace{1cm}
    \caption{General hyperparameters used for all \textit{Sawyer} experiments.} 
    \label{tab:hyparam_3}
\end{table*}

\begin{table*}[h]
    \centering
    \begin{tabular}{c|c|c|c}
        \toprule
        \textbf{Hyperparameter} & \textbf{\textit{SawyerDoorHook}} & \textbf{\textit{SawyerPickup}} & \textbf{\textit{SawyerPush}} \\ 
        \midrule
        Episode length & 100 & 50 & 50 \\
        VAE latent dimension size & 16 & 16 & 4 \\ 
        $\beta$ for $\beta$-VAE & 20 & 30 & 20 \\
        $\alpha$ for Skew & -0.5 & -1 & -1 \\
        \bottomrule
    \end{tabular}
    \hspace{1cm}
    \caption{Specific hyperparameters for the \textit{Sawyer} experiments.} 
    \label{tab:hyparam_4}
\end{table*}



\section{Additional Experiments}
\label{D_Add_Exps}
\subsection{Ablation Study: Ensemble Size}
To study the robustness of the ensemble size for quantifying value uncertainty, we compare the ensemble size of 3, 5, and 7 in \textit{PointMazeSquareLarge} whose results are shown in Figure \ref{fig:ablation_ensemble}.
As illustrated in Figure \ref{fig:ablation_ensemble}, the performance of VUVC with different ensemble size does not differ substantially.

\begin{figure}[h]
    \centering
    \includegraphics[width=0.4\linewidth]{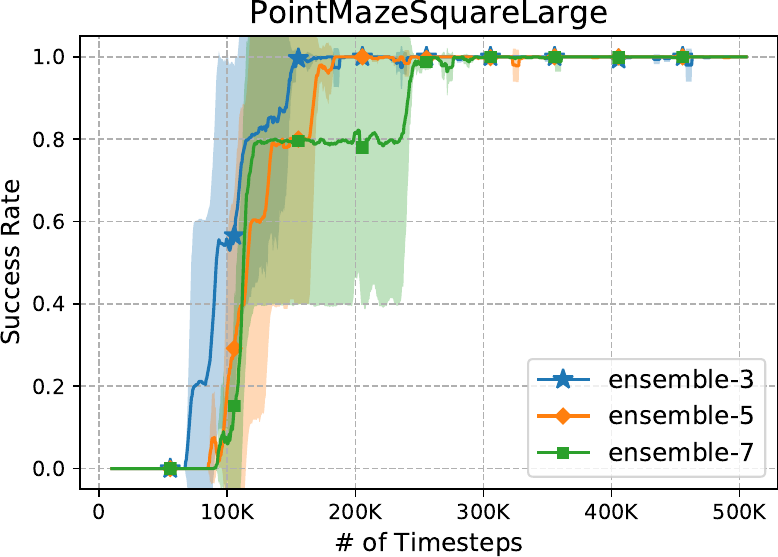}
    \caption{Learning curves for configuration-based point maze navigation task when the ensemble size is 3, 5, and 7, respectively. \textit{Mean (SD)} of success rate over 5 random seeds are reported.}
    \label{fig:ablation_ensemble}
\end{figure}

\subsection{Real-World Robot Experiments} \label{sec:real_exp_setup}

\paragraph{Setup}

The training of our navigation policy is performed using an OpenAI-gym-compatible simulator that we specially design to integrate it into the robot operating system (ROS). We generate an indoor map \cite{kastner2022all} of size 25m $\times$ 25m, as shown in Figure \ref{fig:robot_nav_sim}, and simulate a robot with a 360$^{\circ}$ field of view 2D LiDAR sensor with a resolution set to 512.\footnote{Simulator is available at \href{https://github.com/leekwoon/nav-gym}{\url{github.com/leekwoon/nav-gym}}} We also pre-define a collision threshold, where the agent would be notified of a collision if the sensor measurement is within the threshold. The navigation policy samples an action $a_t \in \mathbb{R}^2$, consisting of linear velocity $v_t \in [0.0, 0.5]$ and angular velocity $w_t \in [-0.64, 0.64]$, at 5Hz. 


We illustrate an example of how VUVC can be used to assist a 2D mobile robot with its navigation task using a laser sensor whose training environment and results are shown in Figure \ref{fig:robot_nav_sim}. By using VUVC, the robot is able to extensively cover the state space, eventually reaching the full state space that would have otherwise been impossible to achieve without detouring through long-range navigation.
In this environment, our method successfully discovers skills  which enable a robot to complete a real-world navigation task in a zero-shot setup, and which can be incorporated with a global planner that further expands reachable distance.

To evaluate the performance of the trained policy on a real robot, we deploy skills on a real Husky A200 mobile robot in the building which is depicted in the left of Figure \ref{fig:real_exp_setup}. The local goal $g_t$ is selected to be $d_{\mathrm{local}}$ away from the robot on the trajectory generated by the global planner, which utilizes the A* algorithm, at 2Hz. To clarify, we update a global plan from the A* algorithm where $g_t$ is selected, and pass $g_t$ to the goal-conditioned policy. We evaluate the performance of the trained policy in the simulator using two key metrics: 1) the robot’s success in reaching the target goal and 2) the time required to traverse to it, as shown at the bottom of Figure \ref{fig:real_exp_result}. We compare these metrics for skills with and without a global planner. For navigation skills without a global planner, a fixed target goal is given, while for skills with a global planner, $g_t$ is given. In a zero-shot setup, the robot is able to successfully navigate to two target goals (red and green stars) located 13m and 22m away from its initial position, respectively, without the use of a global planner. When assisted by a global planner, the robot is able to reach a farthest goal (a blue star) located 31m away from the initial position. Notably, the traverse time to the closest and intermediate goals are reduced from 43 seconds to 33 seconds and from 101 seconds to 67 seconds, respectively. These results demonstrate the effectiveness of our approach in the real-world.

\begin{figure*}[h]
    \centering
    \includegraphics[width=0.9\linewidth]{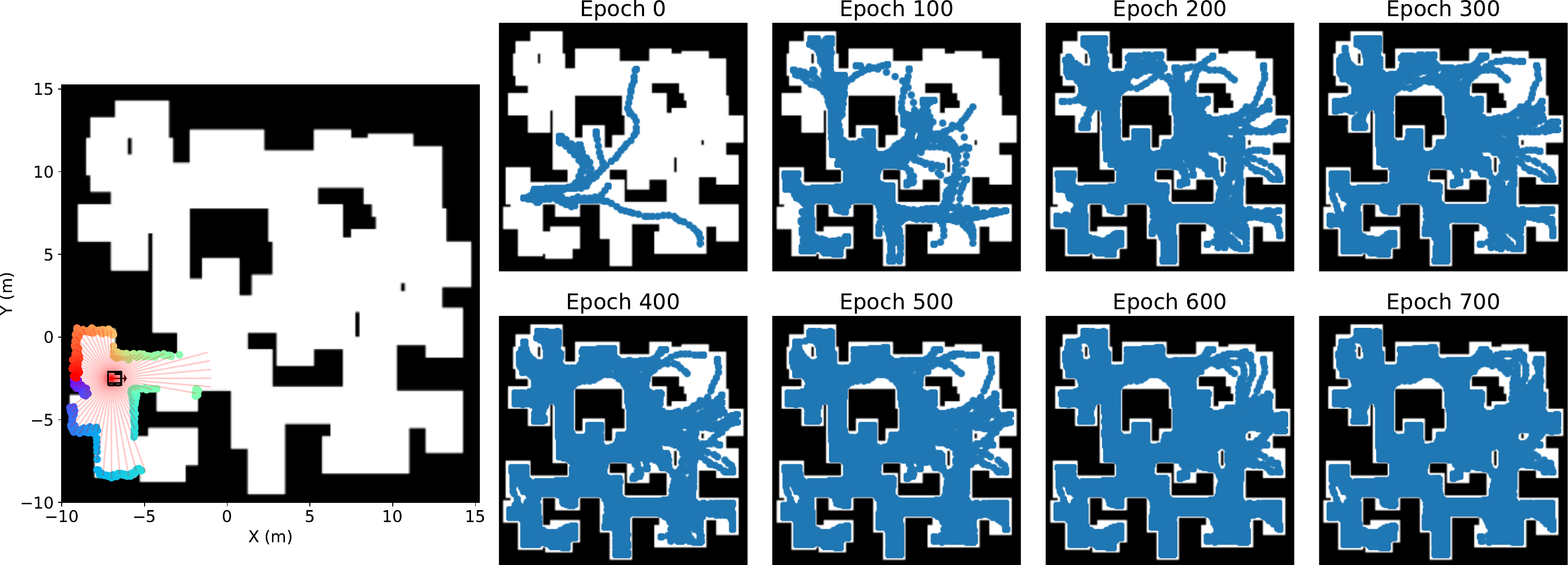}
    \caption{A simulation environment for training a mobile robot (left) and accumulated visited states for every 100 epoch (right). We illustrate an example of how VUVC can be used to assist a 2D mobile robot with its navigation task using a laser sensor. By using VUVC, the robot is able to cover more of the state space, eventually reaching the full state space that would have otherwise been impossible to achieve without detouring through long-range navigation.}
    \label{fig:robot_nav_sim}
\end{figure*}

\clearpage
\subsection{Additional Results}
\label{D_3_Add_Res}
Figure \ref{fig:a_density}$\sim$\ref{fig:sawyerpush_density} demonstrate the curriculum goal distribution and how it changes in the point navigation environments as well as in the robotic manipulation environments which we omit due to space constraint.



\begin{figure}[h]
    \centering
    \includegraphics[width=0.57\linewidth]{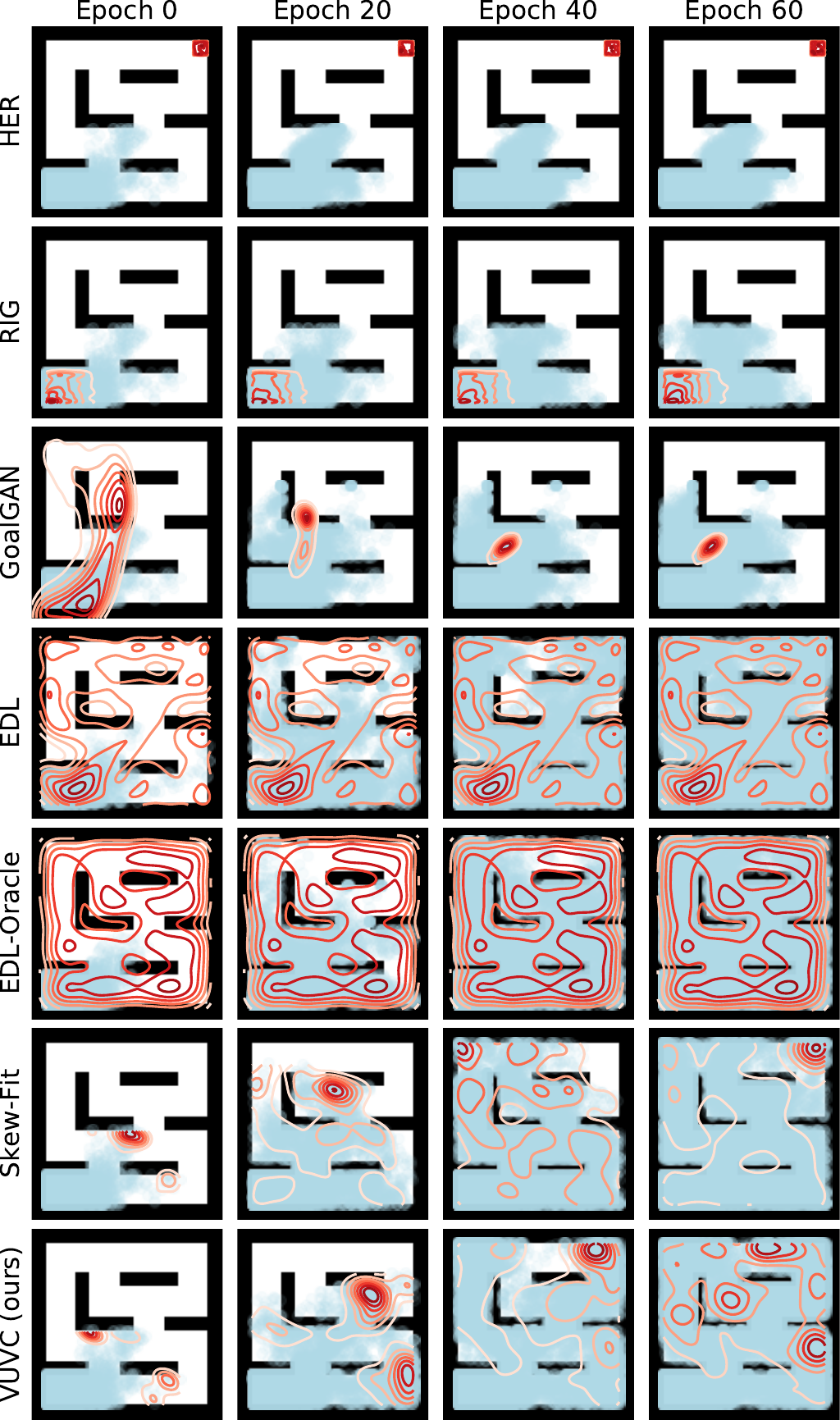}
    \caption{(\textit{PointMazeA}) Curriculum goal distribution and accumulated visited states for a fixed seed for each method. The red contour line illustrates the curriculum goal distribution and cyan dots represent visited states by the agent.}
    \label{fig:a_density}
\end{figure}

\begin{figure}
    \centering
    \includegraphics[width=0.57\linewidth]{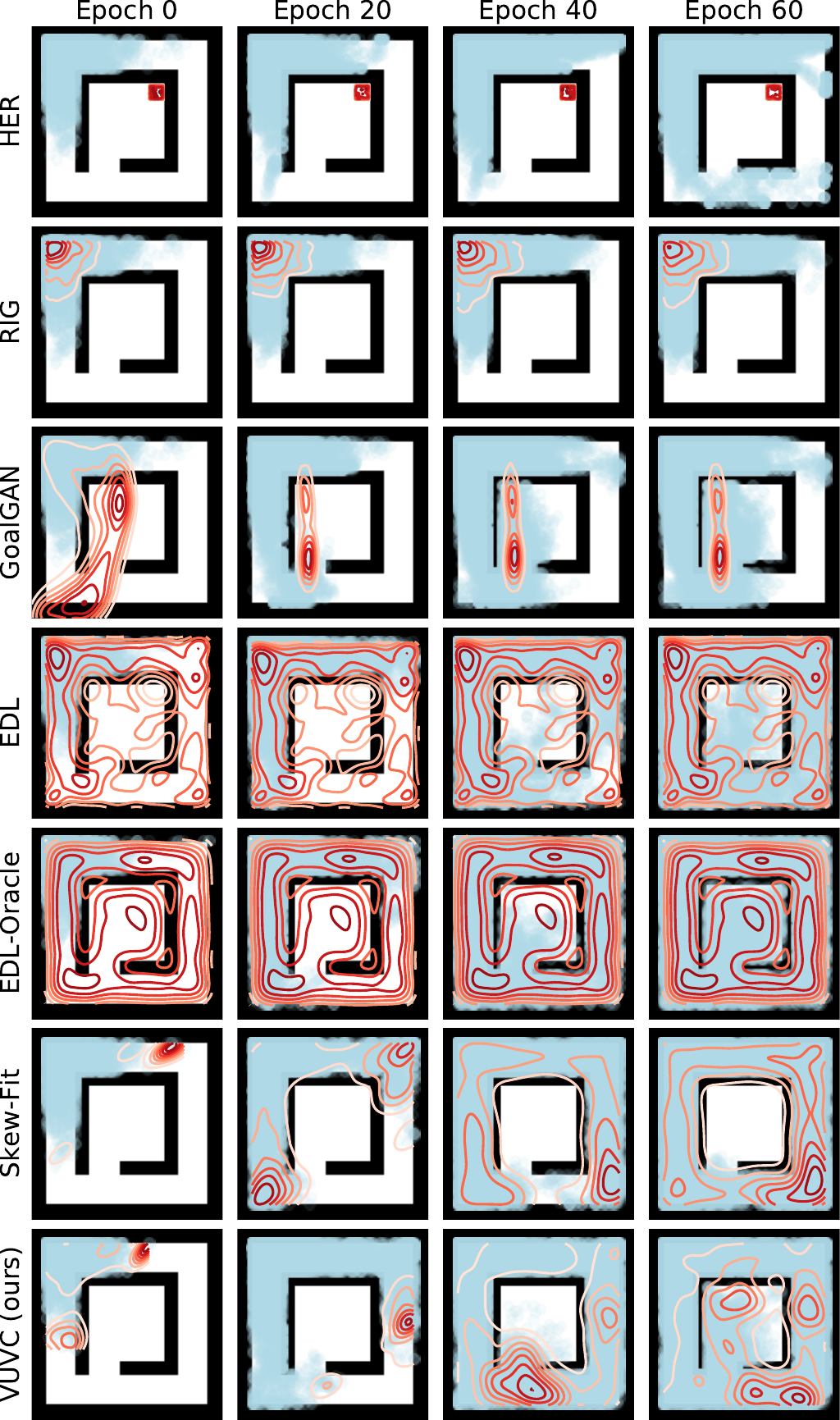}
    \caption{(\textit{PointMazeB}) Curriculum goal distribution and accumulated visited states for a fixed seed for each method. The red contour line illustrates the curriculum goal distribution and cyan dots represent visited states by the agent.}
    \label{fig:b_density}
\end{figure}

\begin{figure}
    \centering
    \includegraphics[width=0.57\linewidth]{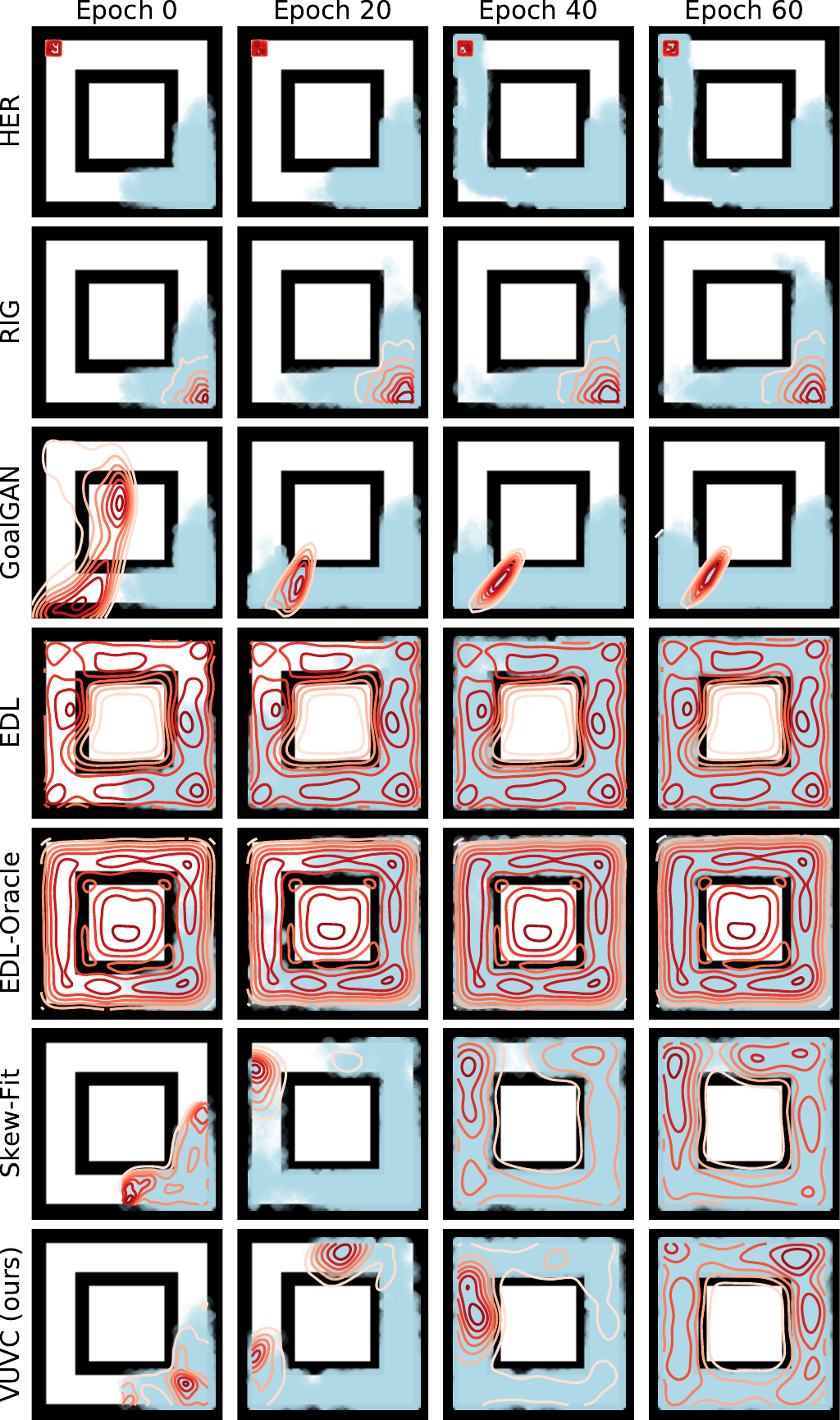}
    \caption{(\textit{PointMazeC}) Curriculum goal distribution and accumulated visited states for a fixed seed for each method. The red contour line illustrates the curriculum goal distribution and cyan dots represent visited states by the agent.}
    \label{fig:c_density}
\end{figure}

\begin{figure}
    \centering
    \includegraphics[width=0.9\linewidth]{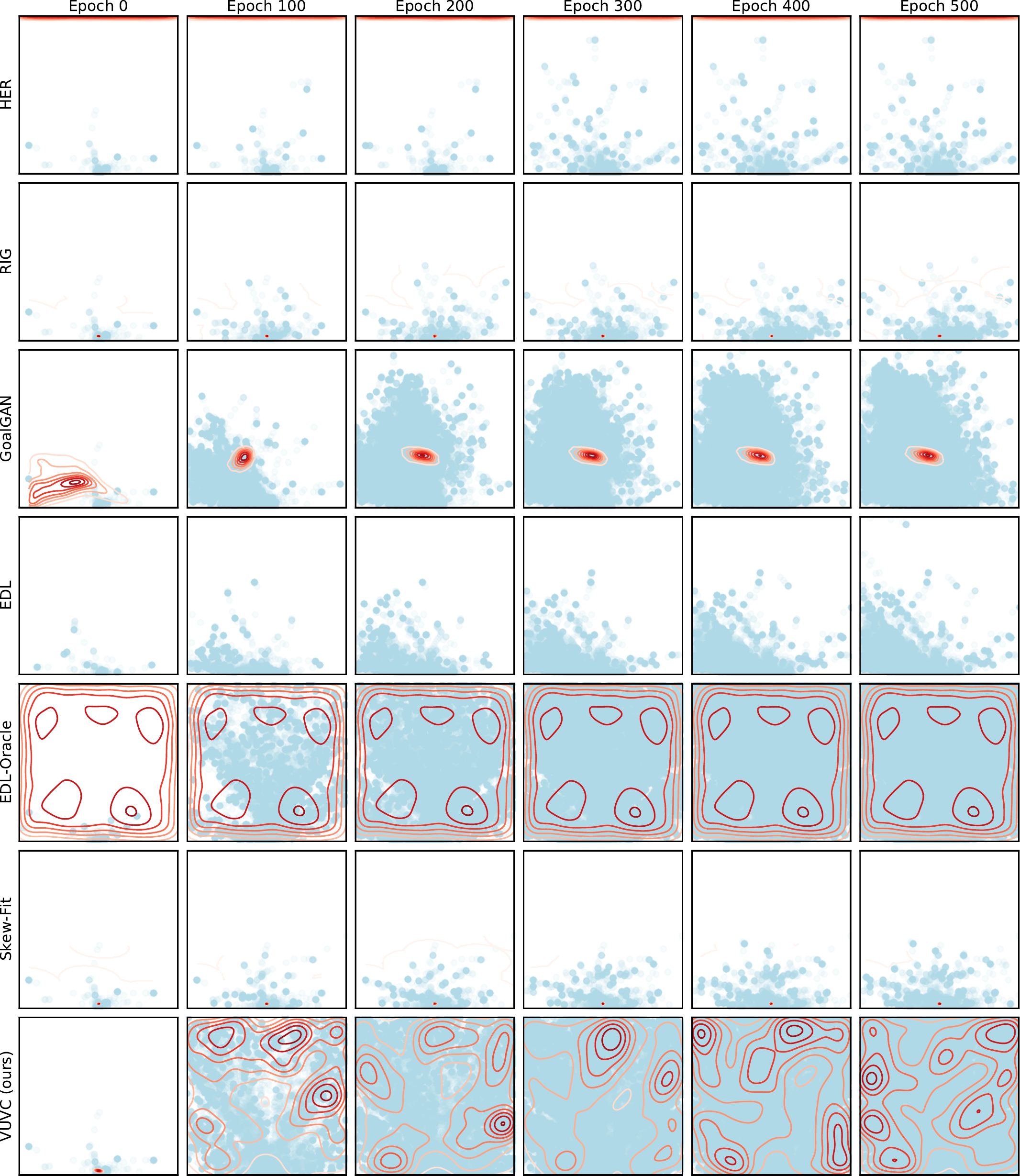}
    \caption{(\textit{FetchPush}) Curriculum goal distribution and accumulated visited states for a fixed seed for each method. The red contour line illustrates the curriculum goal distribution and cyan dots represent visited states by the agent.}
    \label{fig:fetchpush_density}
\end{figure}

\begin{figure}
    \centering
    \includegraphics[width=0.9\linewidth]{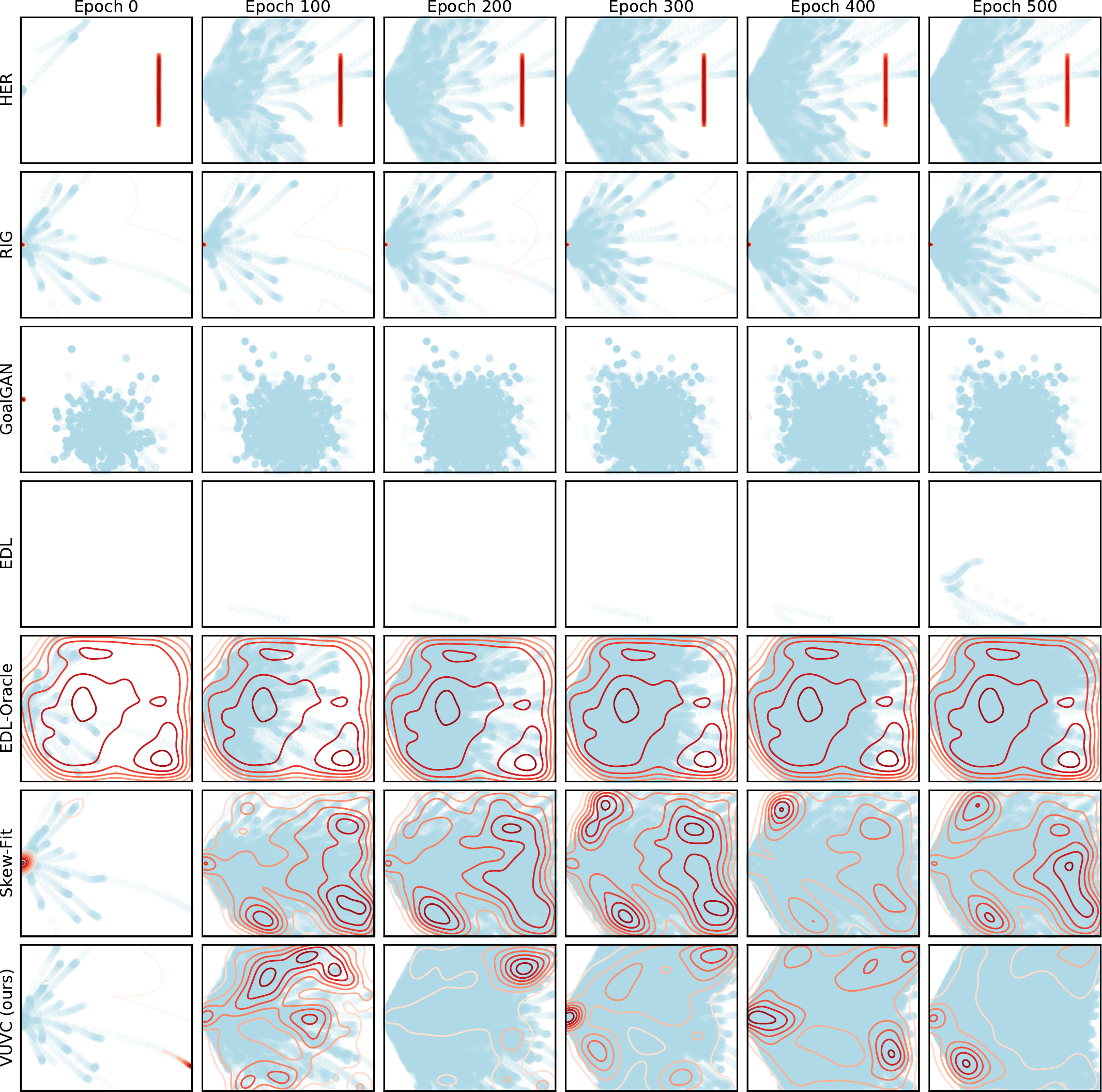}p
    \caption{(\textit{FetchSlide}) Curriculum goal distribution and accumulated visited states for a fixed seed for each method. The red contour line illustrates the curriculum goal distribution and cyan dots represent visited states by the agent.}
    \label{fig:fetchslide_density}
\end{figure}

\begin{figure}
    \centering
    \includegraphics[width=0.9\linewidth]{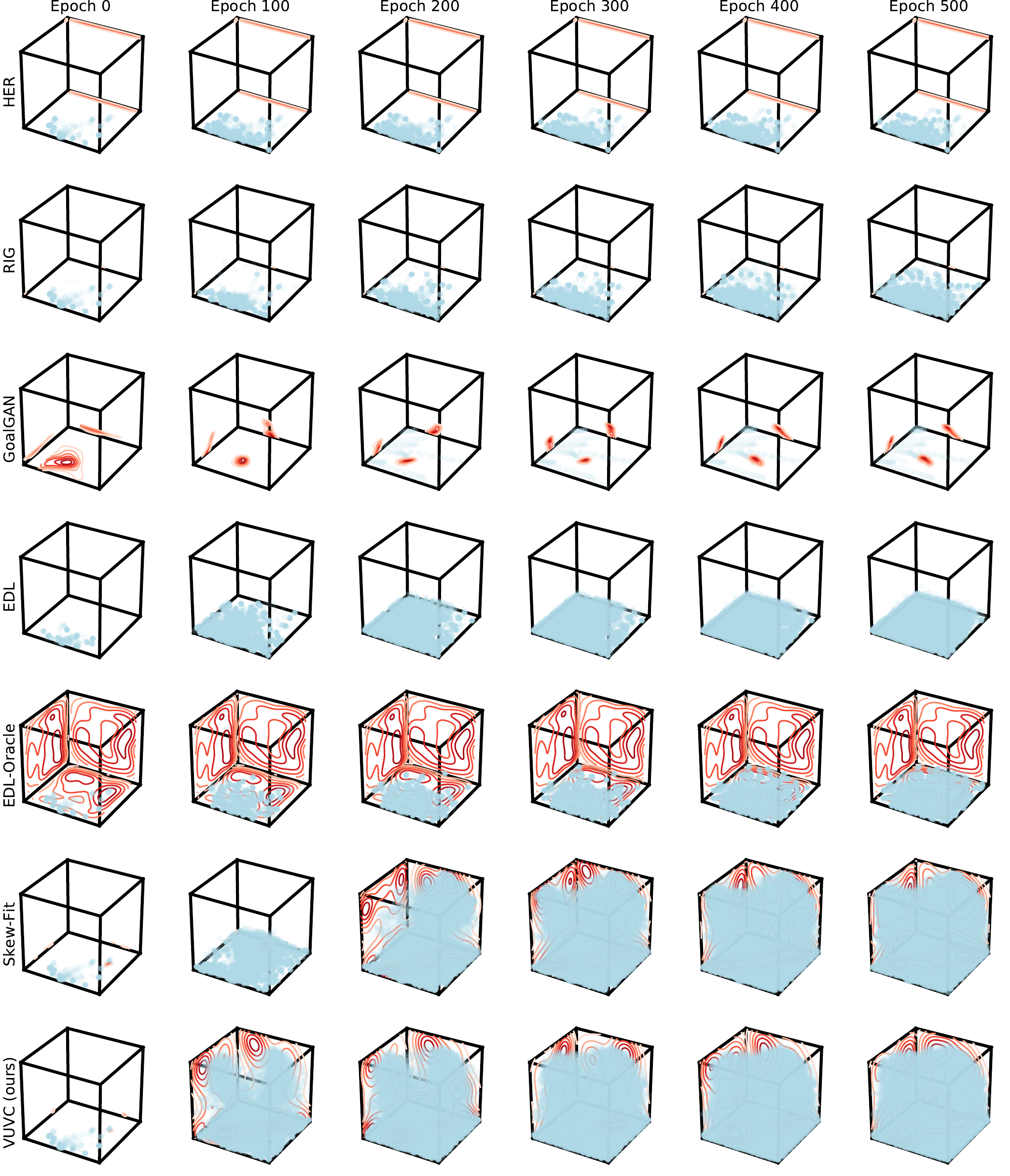}
    \caption{(\textit{FetchPickAndPlace}) Curriculum goal distribution and accumulated visited states for a fixed seed for each method. The red contour line illustrates the curriculum goal distribution and cyan dots represent visited states by the agent.}
    \label{fig:fetchpickandplace_density}
\end{figure}

\begin{figure}
    \centering
    \includegraphics[width=0.9\linewidth]{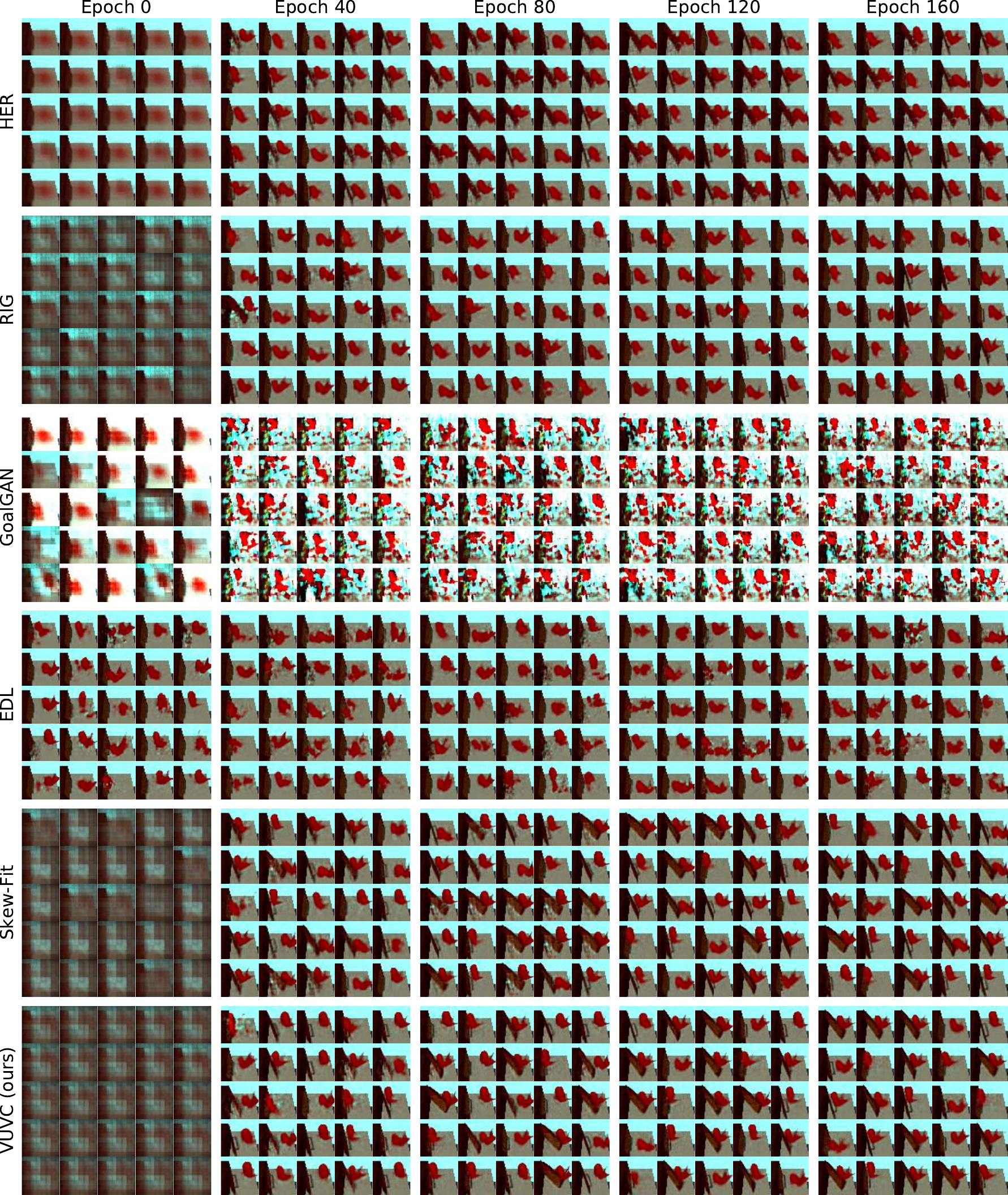}
    \caption{(\textit{SawyerDoorHook}) Examples of curriculum goals for a fixed seed for each method. Latent codes are given as a curriculum goal and their reconstructed images are illustrated for visualization.}
    \label{fig:sawyerdoor_density}
\end{figure}

\begin{figure}
    \centering
    \includegraphics[width=0.8\linewidth]{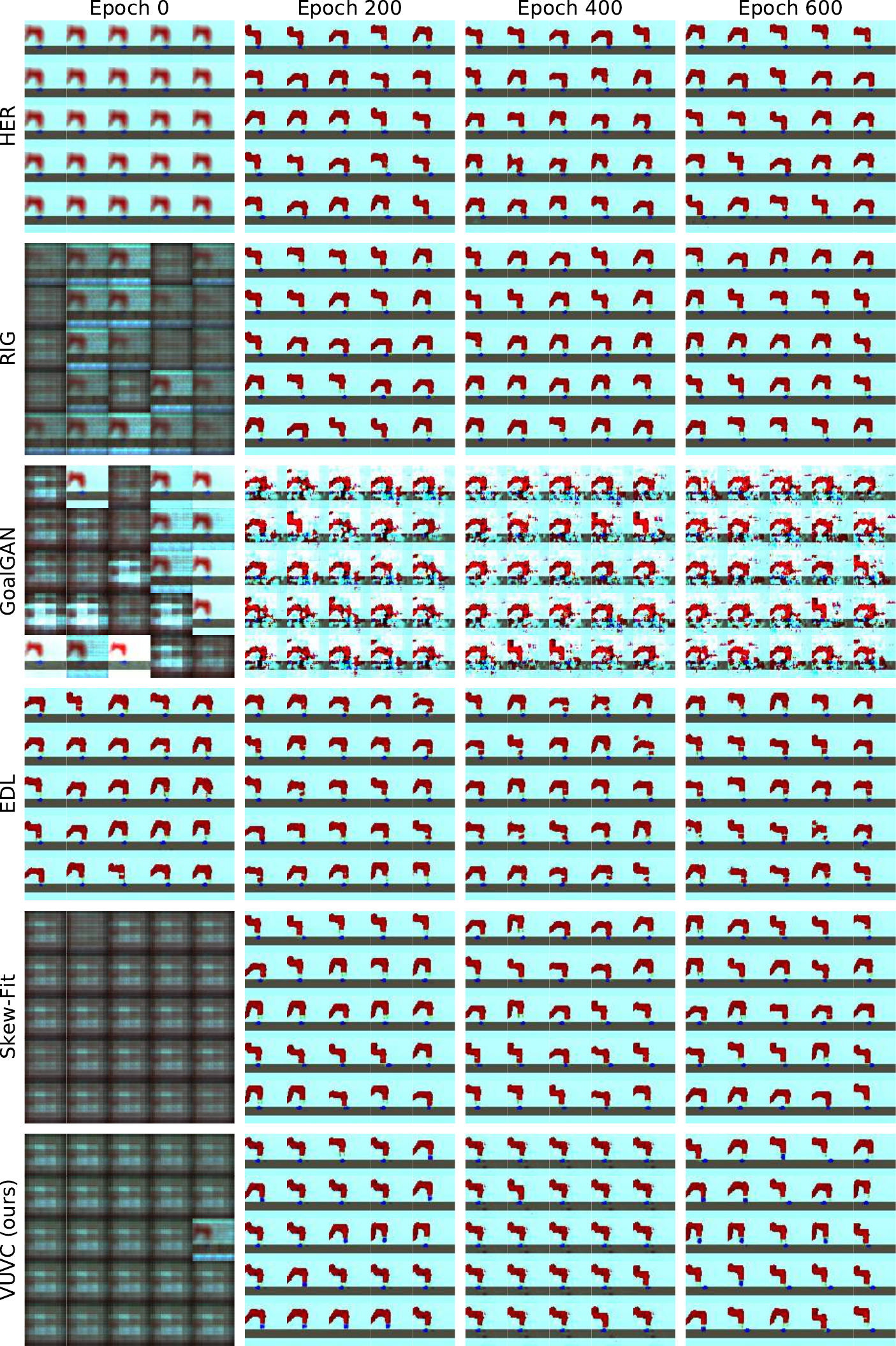}
    \caption{(\textit{SawyerPickup}) Examples of curriculum goals for a fixed seed for each method. Latent codes are given as a curriculum goal and their reconstructed images are illustrated for visualization.}
    \label{fig:sawyerpickup_density}
\end{figure}

\begin{figure}
    \centering
    \includegraphics[width=0.9\linewidth]{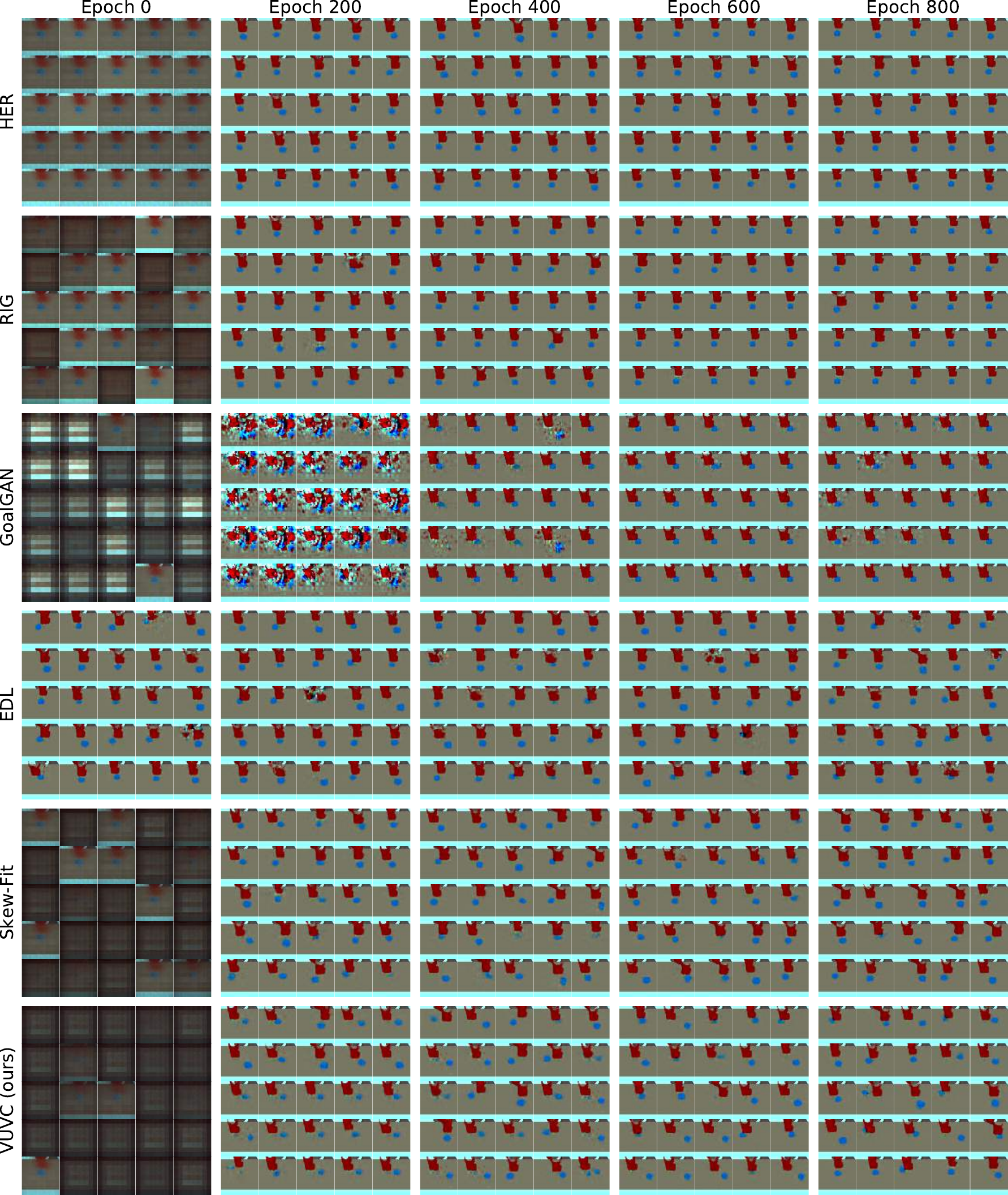}
    \caption{(\textit{SawyerPush}) Examples of curriculum goals for a fixed seed for each method. Latent codes are given as a curriculum goal and their reconstructed images are illustrated for visualization.}
    \label{fig:sawyerpush_density}
\end{figure}




%



\end{document}